\documentclass{article}

\usepackage{xurl}
\usepackage{todonotes}
\usepackage{tabularx} 
\usepackage{multicol}
\usepackage{multirow}
\usepackage{comment}
\usepackage{subcaption}

\usepackage{amsmath,amssymb,amsthm}
\usepackage{graphicx}
\usepackage{import}
\usepackage{booktabs}
\usepackage{hyperref}
\usepackage{bm}

\usepackage[accepted]{icml2026}
    \icmltitlerunning{Sensitivity-Conditioned Bernoulli Flow Matching}

\newtheorem{theorem}{Theorem}

\newtheorem{proposition}{Proposition}

\newtheorem{definition}{Definition}
\newtheorem{assumption}{Assumption}

\newcommand{\R}{\mathbb{R}}
\newcommand{\E}{\mathbb{E}}

\newcommand{\Vol}{\mathcal{V}}
\newcommand{\Obj}{\mathcal{J}}
\newcommand{\Res}{\mathcal{R}}

\newcommand{\vv}{\bm{v}}
\newcommand{\ww}{\bm{w}}

\DeclareMathOperator{\CE}{CE}

\newcommand{\mohammadd}[1]{\textcolor{black}{ #1}}
\newcommand{\duarte}[1]{\textcolor{black}{ #1}}

\newcommand{\yunjia}[1]{\textcolor{black}{ #1}}
\newcommand{\nils}[1]{\textcolor{black}{#1}}
\newenvironment{mohammad}
  {\begingroup\color{black}}
  {\endgroup}

\begin{document}

\twocolumn[
\icmltitle{On the Generalization in Topology Optimization via Sensitivity-Conditioned Bernoulli Flow Matching}
\begin{icmlauthorlist}
\icmlauthor{Mohammad Rashed}{tum,bmw,mcml}
\icmlauthor{Duarte F. Valoroso Madeira}{bmw,tuhh}
\icmlauthor{Babak Gholami}{bmw}
\icmlauthor{Caglar Guerbuez}{bmw}
\icmlauthor{Yunjia Yang}{tum}
\icmlauthor{Nils Thuerey}{tum,mcml}
\end{icmlauthorlist}

\icmlaffiliation{tum}{School of Computation, Information and Technology, Technical University of Munich, Germany}
\icmlaffiliation{bmw}{BMW AG, Munich, Germany}
\icmlaffiliation{tuhh}{Hamburg University of Technology, Hamburg, Germany}
\icmlaffiliation{mcml}{Munich Center for Machine Learning, Munich, Germany.}
\icmlcorrespondingauthor{Mohammad Rashed}{m.rashed@tum.de}

\vskip 0.3in
]

\printAffiliationsAndNotice{}

\begin{abstract}

Surrogate models for topology optimization (TO) exhibit highly variable out-of-distribution (OOD) generalization under distribution shifts such as changing loads or boundary conditions, yet the source of this variability remains unclear. We hypothesize that OOD performance is governed by how much information the conditioning signal preserves about the adjoint sensitivity (reduced gradient) that drives classical TO. Modeling the TO pipeline as a causal Markov chain, the Data Processing Inequality establishes that, under this abstraction, the sensitivity field is an information-theoretically optimal conditioning signal for topology prediction. However, computing exact adjoint sensitivities can be expensive or unavailable in practice; we observe that certain physical fields can approximate sensitivities through monotone transformations. To formalize this, we introduce \textbf{pseudo-sensitivities} to characterize which fields enable generalization versus those that are information-poor. We then show that a sensitivity-conditioned Bernoulli flow-matching generator empirically confirms these predictions: conditioning on sensitivities yields state-of-the-art OOD performance, while increasingly distant physical fields degrade toward raw parameter conditioning. Results hold across structural TO benchmarks under load shifts and our new CFD-TO dataset under boundary-condition shifts such as multi-outlet configurations. Code and datasets are available at \url{https://github.com/tum-pbs/topotransformer}.
\end{abstract}

\section{Introduction}
\begin{figure}[t]
\centering
\begin{subfigure}[b]{1\columnwidth}
\includegraphics[width=\columnwidth]{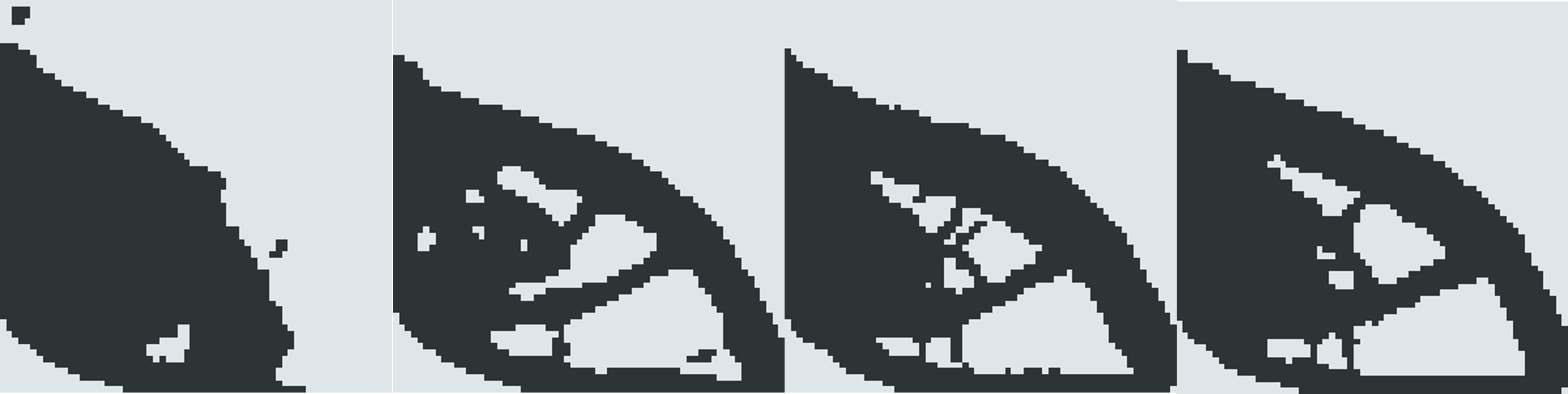}
\end{subfigure}
\begin{subfigure}[b]{1\columnwidth}
\includegraphics[width=\columnwidth]{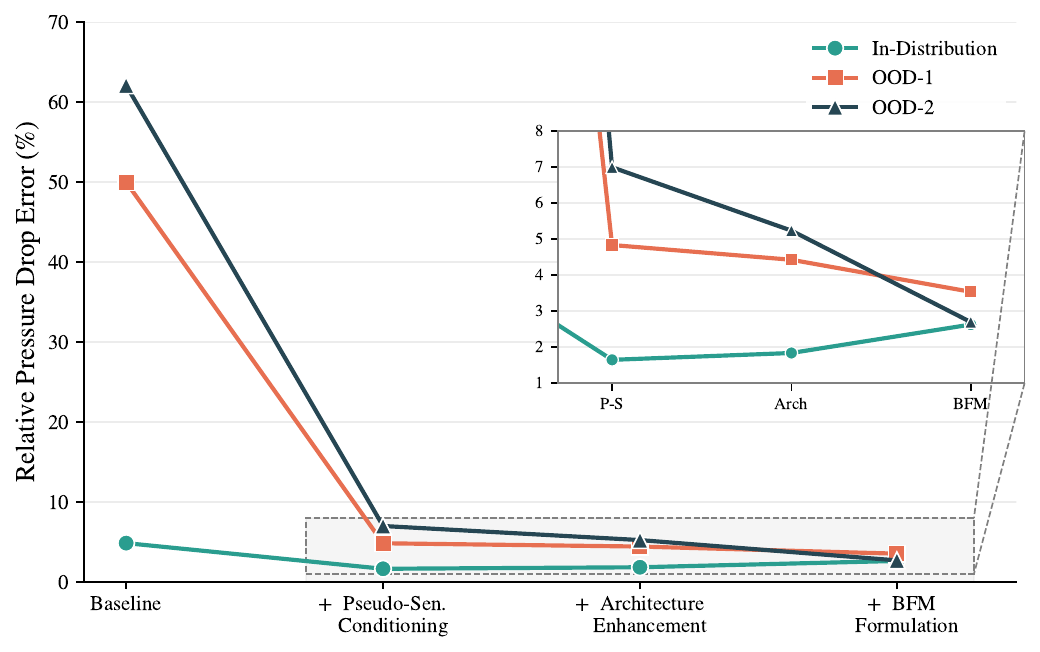}
\end{subfigure}
\caption{Conditioning signal governs OOD generalization. \textit{Bottom:} Quantitative CFD results reveal that the most significant performance gain comes from pseudo-sensitivity conditioning, empirically validating our theoretical finding that sensitivity-aligned signals are information-theoretically optimal. Subsequent architectural changes (such as Bernoulli Flow Matching) provide secondary, gradual refinements. \textit{Top:} Qualitative predictions mirror this trend: the conditioning signal establishes the correct global topology and connectivity, while additional model modules primarily enhance boundary sharpness. This highlights that the proximity of the conditioning signal to the true sensitivity is the dominant factor in solving out-of-distribution topology problems.}
\label{fig:teaser}
\end{figure}

Topology optimization seeks an optimal material distribution\duarte{, by typically} minimizing an objective functional subject to PDE constraints~\cite{bendsoe1988generating,sigmund2013topology}. Classical gradient-based methods rely on repeated forward and adjoint PDE solves, often requiring hundreds of iterations. This computational burden is especially severe for CFD problems involving turbulent flows, where each adjoint solve may be unstable or require simplifying assumptions such as frozen turbulence.

Recent learning-based approaches aim to amortize this cost by predicting optimal topologies directly from problem specifications. Deterministic regressors conditioned on boundary conditions and loads achieve fast inference but exhibit limited out-of-distribution (OOD) generalization \cite{nobari2024nito}. Generative models such as TopoDiffusion~\cite{maze2022topodiff} offer improved robustness by conditioning on strain energy density along with other physical fields, though the mechanism underlying this improvement remains unclear. Moreover, existing generative approaches rely on continuous Gaussian diffusion \cite{maze2022topodiff, DBLP:journals/corr/abs-2303-09760}, a choice that is fundamentally misaligned with binary topology representations.

\yunjia{Another} notable gap in the literature is that all existing learning-based topology optimization methods target structural compliance problems. To our knowledge, no prior work has applied deep generative models to turbulent CFD topology optimization, likely due to the expense of generating training data and the complexity of computing adjoint sensitivities under RANS turbulence closures.

In this work, we argue that while architectural innovations can enhance generalization, the dominant factor governing out-of-distribution performance is the choice of conditioning signal. Modeling the optimization pipeline as a Markov chain, we apply the Data Processing Inequality to show that adjoint sensitivities are information-theoretically optimal for predicting optimal topologies. We further introduce pseudo-sensitivities (physical fields related to true sensitivities via monotone transformations) which explains why strain energy density succeeds for structural problems\yunjia{,} and identifies velocity magnitude as the analogous signal for CFD.
To this end, we \nils{make the following contributions}:
\vspace{-10pt}
\begin{itemize}
    \item We show that adjoint sensitivities minimize conditional entropy for topology prediction (Theorem~\ref{thm:sensitivity_dpi}), providing a principled explanation for OOD generalization.
    \item We formalize pseudo-sensitivities (Definition~\ref{def:pseudo_sensitivity}) and derive closed-form expressions for structural compliance and CFD energy dissipation, unifying prior empirical conditioning choices under adjoint theory.
    \item We present the first learning-based approach for turbulent CFD topology optimization, releasing a dataset of 10,000 problems with multi-outlet OOD splits.
    \item We propose a sensitivity-conditioned Bernoulli Flow Model that respects the binary nature of topology, demonstrating improved OOD generalization over continuous diffusion baselines.
\end{itemize}

\section{Related Work}

\paragraph{Classical Topology Optimization.}
The Solid Isotropic Material with Penalization (SIMP) method~\cite{bendsoe1989optimal,zhou1991coc,mlejnek1992genesis} serves as the foundation for topology optimization. SIMP relaxes the binary material distribution $\rho \in \{0,1\}^N$ into a continuous density field $\rho \in [0,1]^N$, with penalization to encourage binary solutions. Optimization techniques such as the optimality criteria method~\cite{bendsoe2004topology} and the method of moving asymptotes~\cite{svanberg1987mma}, typically used with SIMP, rely on adjoint sensitivity analysis to efficiently compute gradients in these high-dimensional design spaces. And while these classical methods achieve high-quality solutions, they require solving \duarte{hundreds} of expensive PDEs, motivating learning-based acceleration.

\paragraph{Learning-Based Topology Optimization.}
Early learning approaches~\cite{sosnovik2019neural,cang2019improving} used neural networks to accelerate iterative optimization by predicting intermediate topologies or gradients. More recent methods aim for direct, non-iterative prediction. Neural Implicit Topology Optimization (NITO)~\cite{nobari2024nito} employs implicit neural representations for resolution-free topology generation, achieving significant speedups over SIMP with 80\% lower compliance errors than prior methods. However, deterministic approaches conditioned solely on boundary conditions and loads exhibit limited OOD generalization.

Generative models offer improved robustness through learned stochastic mappings. TopoDiffusion~\cite{maze2022topodiff} introduced diffusion models for topology optimization, conditioning on strain/stress energy density and employing surrogate-guided sampling to minimize compliance and ensure manufacturability. It demonstrated 8$\times$ reduction in OOD compliance error over conditional GANs~\cite{kallioras2020accelerated,nie2021topologygan}. Despite empirical success, theoretical understanding of why these models generalize remains limited.

\paragraph{Adjoint Methods in CFD and Topology Optimization.}
Adjoint sensitivity analysis~\cite{jameson1988aerodynamic,giles2000introduction,giles1997adjoint} is foundational for gradient-based design optimization in fluid mechanics. For turbulent flows, the treatment of turbulence model equations in the adjoint system is critical~\cite{giles2000introduction,marta2013handling}. The ``frozen turbulence'' assumption~\cite{marta2013handling}, which neglects variations of eddy viscosity with respect to design variables, significantly simplifies implementation but can produce wrongly-signed sensitivities in turbulent regions~\cite{kuehl2021adjoint}. Recent work addresses this through exact discrete adjoints~\cite{dilgen2018topology} and wall-function consistent formulations~\cite{palacios2013hybrid}. To our knowledge, no prior work has applied learning-based generative models to turbulent RANS topology optimization.

\paragraph{Generative Models: Diffusion and Flow Matching.}
Denoising diffusion probabilistic models (DDPM)~\cite{ho2020denoising,nichol2021improved,dhariwal2021diffusion} have achieved state-of-the-art image generation by learning to reverse a gradual noising process. Conditional variants~\cite{ho2021classifierfree,dhariwal2021diffusion} enable control through classifier-free guidance. However, DDPMs model continuous Gaussian distributions, which are fundamentally misaligned with binary topology representations.

Flow matching~\cite{lipman2023flow} provides a simulation-free alternative for training continuous normalizing flows (CNFs) by regressing vector fields along probability paths. It subsumes diffusion as a special case while enabling more efficient optimal transport paths. Recent extensions address discrete and categorical data. Discrete diffusion models~\cite{austin2021structured,hoogeboom2021argmax} handle categorical distributions through transition matrices. Most relevant to our work, Bernoulli Flow Models (BFM)~\cite{gat2024bernoulli} introduce flow matching for binary data by operating in Bernoulli parameter space with closed-form posteriors, avoiding invalid parameters and model collapse. 

\section{Preliminaries and Problem Formulation}

\subsection{Topology Optimization Formulation}
Consider a design domain $\Omega \subset \R^d$. We aim to find a density field $\rho \in L^\infty(\Omega)$ that minimizes an objective $\Obj$, subject to a volume constraint $\Vol_{req}$ and a PDE constraint, with
\[
0 \le \rho(x) \le 1 \quad \text{for almost every } x \in \Omega.
\]
Let $\mathcal{U}$ be the state space (e.g., $H^1_0(\Omega)$). The optimization problem is:
\begin{equation}
\label{eq:optimization}
\begin{aligned}
    \min_{\rho} \quad & \Obj(u(\rho), \rho) \\
    \text{s.t.} \quad & \Res(u, \rho) = 0 \quad (\text{State Equation}) \\
    & \int_\Omega \rho \, dx \le \Vol_{req}, \quad 0 \le \rho(x) \le 1.
\end{aligned}
\end{equation}
where $\Res: \mathcal{U} \times L^\infty(\Omega) \to \mathcal{U}^*$ is the PDE residual and $u(\rho)$ denotes the state solving the PDE for a given design $\rho$.

\subsection{The Reduced Gradient (Sensitivity)}
To solve \eqref{eq:optimization}, one typically employs the adjoint method. Let $\lambda \in \mathcal{U}$ be the adjoint state satisfying:
\begin{equation}
    \label{eq:adjoint}
    \left(\frac{\partial \Res}{\partial u}\right)^* \lambda = \frac{\partial \Obj}{\partial u}.
\end{equation}
The \textbf{Sensitivity Field} $S \in L^\infty(\Omega)$ is defined as the total derivative (reduced gradient) of the objective with respect to $\rho$:
\begin{equation}
    \label{eq:sensitivity}
    S(x) := \frac{d\Obj}{d\rho}(x) = \frac{\partial \Obj}{\partial \rho}(x) - \left\langle \lambda, \frac{\partial \Res}{\partial \rho}(\cdot, x) \right\rangle.
\end{equation}
We write $S(\rho)$ when we wish to emphasize the dependence on the design $\rho$, and denote by 
\[
S_0 := S(\rho_{\mathrm{init}})
\]
the \emph{initial} sensitivity field evaluated at the initial design $\rho_{\mathrm{init}}$. Abusing notation, we use the same symbols ($\Theta, X, S_0, \rho^\star$) to denote both random variables and their realizations; the distinction will be clear from context.
\begin{figure*}[t]
\centering
\includegraphics[width=0.9\textwidth]{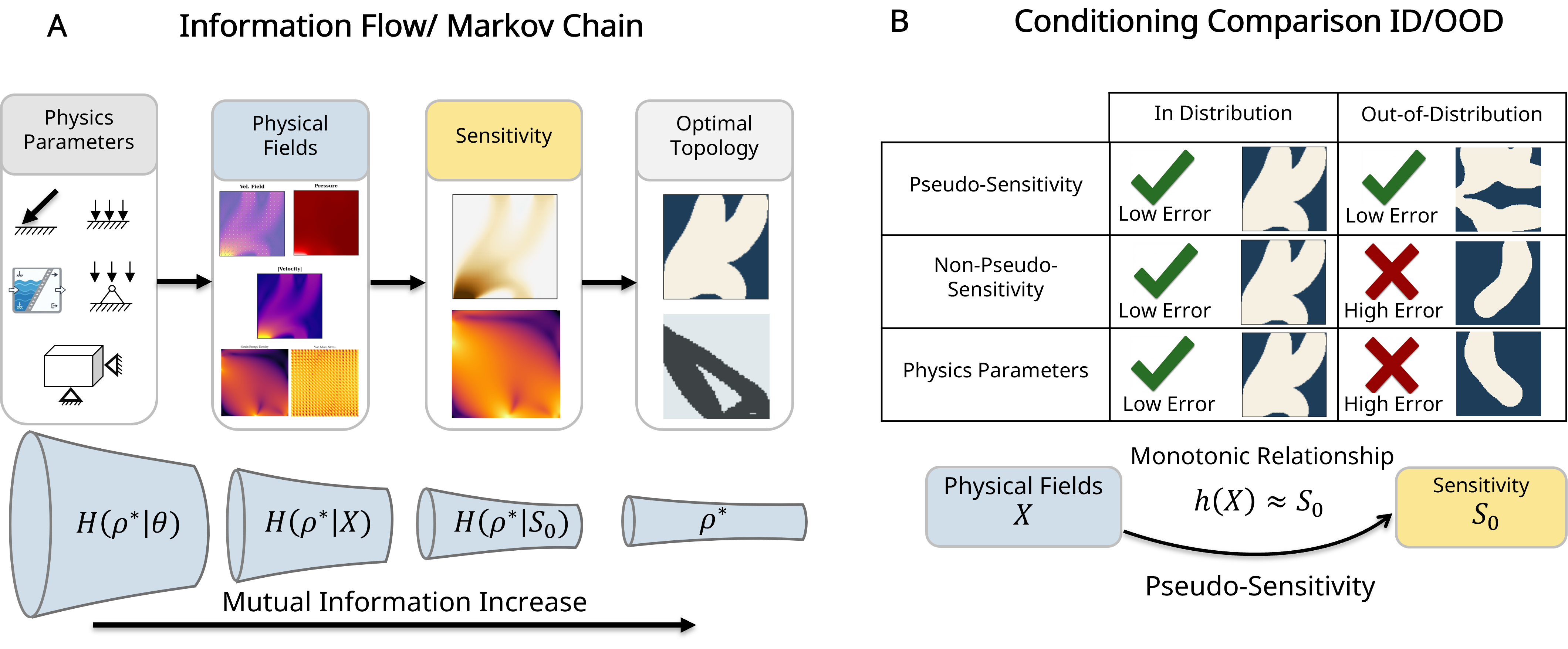}
\caption{Information flow in topology optimization. (A) The TO pipeline forms a Markov chain $\Theta \to X \to S_0 \to \rho^\star$, with conditional entropy decreasing toward the sensitivity field. (B) Conditioning comparison: pseudo-sensitivities generalize OOD, while non-pseudo fields and raw parameters degrade.}
\label{fig:dpi}
\end{figure*}

\section{Theoretical Framework: Conditioning and Generalization in Topology Optimization}

We analyze learning-based topology optimization through the lens of information flow.
Our goal is to understand which conditioning signals fundamentally enable
generalization across problem instances, independent of the particular generative
model used.

Throughout, we consider first-order topology optimization methods (e.g., SIMP) and focus on the \emph{initial sensitivity field}, which empirically governs the
global structure of the final design.

\subsection{A Causal Abstraction of Topology Optimization}

Let $\Theta \sim p(\Theta)$ denote random problem parameters, including boundary
conditions, loads, material properties, and geometry. Let $\rho^\star$ denote the
converged (binary) topology produced by a classical optimizer initialized from a
fixed initial design $\rho_{\mathrm{init}}$.
We model the topology optimization pipeline as the following sequence of operators:
\begin{equation}
\Theta
\;\xrightarrow{\;\mathcal{F}\;}\;
X
\;\xrightarrow{\;\mathcal{A}\;}\;
S_0
\;\xrightarrow{\;\mathcal{T}\;}\;
\rho^\star ,
\label{eq:causal_chain}
\end{equation}
where \nils{$S_0 := \tfrac{d\Obj}{d\rho}\big|_{\rho_{\mathrm{init}}}$ is the
    \textbf{initial reduced gradient (sensitivity field)} computed via an adjoint
    solve. $X$ denotes physical state fields obtained by solving the governing PDE at
    $\rho_{\mathrm{init}}$ (e.g., displacement, velocity, pressure), and $\mathcal{T}$ denotes a deterministic first-order optimizer mapping the
    initial sensitivity landscape to a target topology.}
%
%
This abstraction isolates the \emph{information pathway} through which problem
parameters influence the optimal design.

\subsection{Markov Structure and Modeling Assumptions}

We formalize the abstraction above via the following assumption.

\begin{assumption}[Causal Markov Structure]
\label{assump:markov_theory}
The random variables $(\Theta, X, S_0, \rho^\star)$ satisfy the following conditional
independencies:
\vspace{-2pt}
\begin{enumerate}
    \item $\Theta \rightarrow X \rightarrow S_0$
    \quad (sensitivities are computed deterministically from physical fields);
    \item $X \rightarrow S_0 \rightarrow \rho^\star$
    \quad (the optimizer acts on the sensitivity field and constraints, not directly
    on raw state fields);
    \item $\Theta \rightarrow X \rightarrow \rho^\star$
    \quad (problem parameters influence the design only through the induced physical
    solution).
\end{enumerate}
\end{assumption}
\vspace{-1pt}
Together, these relations define the Markov chain
\begin{equation}
\Theta \;\rightarrow\; X \;\rightarrow\; S_0 \;\rightarrow\; \rho^\star .
\label{eq:markov_chain}
\end{equation}
Figure~\ref{fig:dpi} illustrates this information flow and the resulting entropy ordering.

\subsection{Information-Theoretic Limits of Conditioning Signals}

We now quantify how much information different conditioning signals can provide about
the optimal topology.

\begin{theorem}[Sensitivity as Information-Theoretically Optimal Conditioning]
\label{thm:sensitivity_dpi}
Under Assumption~\ref{assump:markov_theory}, the mutual information between the optimal
topology $\rho^\star$ and candidate conditioning signals satisfies:
\begin{equation}
I(\rho^\star;\Theta)
\;\le\;
I(\rho^\star;X)
\;\le\;
I(\rho^\star;S_0).
\label{eq:mi_ordering}
\end{equation}
Equivalently, the conditional entropies obey:
\begin{equation}
H(\rho^\star \mid S_0)
\;\le\;
H(\rho^\star \mid X)
\;\le\;
H(\rho^\star \mid \Theta).
\label{eq:entropy_ordering}
\end{equation}
\end{theorem}

\begin{proof}
By Assumption~\ref{assump:markov_theory}, the variables form the Markov chain
$\Theta \rightarrow X \rightarrow S_0 \rightarrow \rho^\star$.
Applying the Data Processing Inequality twice yields
$I(\Theta;\rho^\star) \le I(X;\rho^\star) \le I(S_0;\rho^\star)$.
The entropy ordering follows from
$H(\rho^\star \mid Z) = H(\rho^\star) - I(\rho^\star; Z)$.
\end{proof}

\paragraph{Interpretation.}
Theorem~\ref{thm:sensitivity_dpi} does not claim that conditioning on sensitivities
guarantees perfect prediction.
Rather, it establishes a \emph{fundamental lower bound} on irreducible uncertainty:
among all admissible conditioning signals derived from the problem setup, the
sensitivity field minimizes the Bayes-optimal prediction error for $\rho^\star$. The proof is detailed in Appendix \ref{proof:dpi}.
\begin{figure*}[t]
\centering
\includegraphics[width=\textwidth]{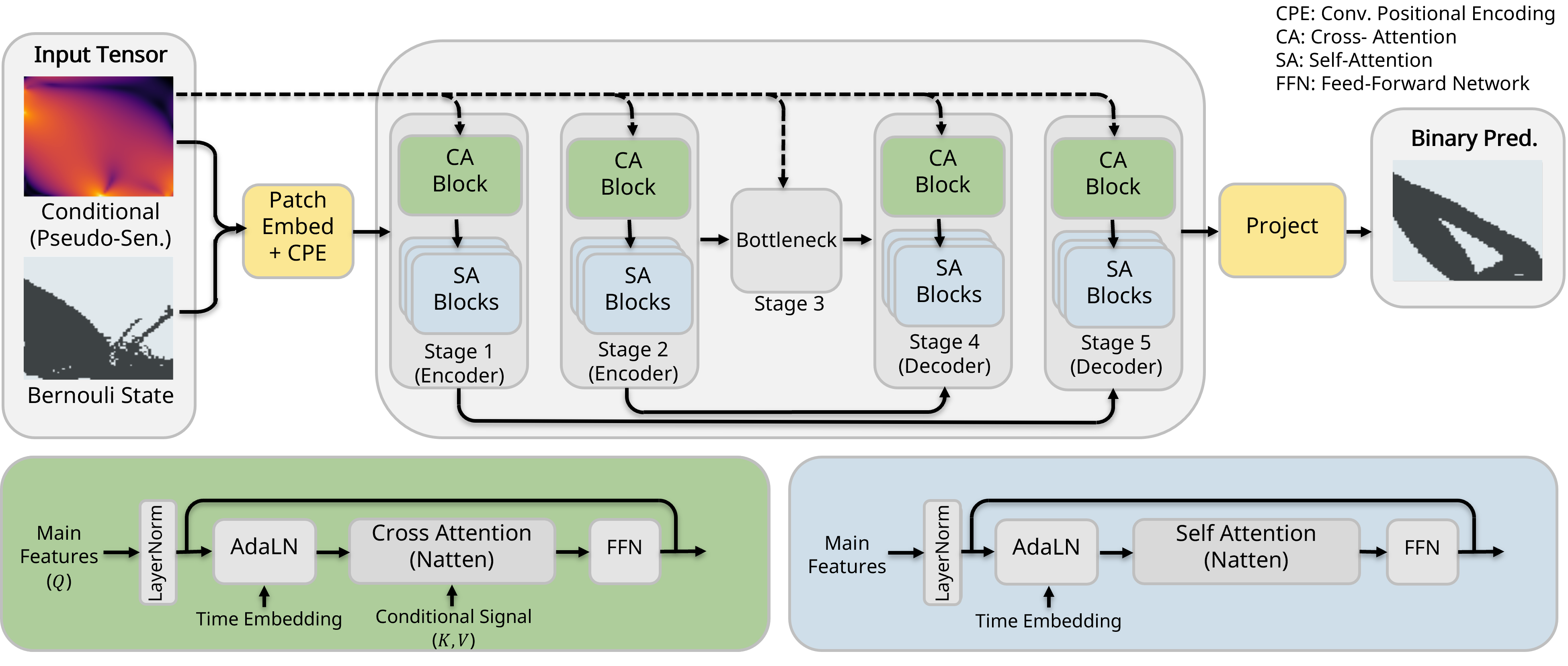}
\caption{Network architecture. A hierarchical vision transformer with encoder-decoder structure processes the conditioning field (sensitivity or pseudo-sensitivity) and noisy Bernoulli state $x_t$. Cross-attention blocks at each resolution allow the sensitivity field to modulate feature updates, while AdaLN injects timestep information in addition to main features.}
\label{fig:arch}
\end{figure*}
\subsection{Pseudo-Sensitivities: When Physical Fields Are Sufficient}

While $S_0$ is optimal, computing exact adjoint sensitivities may be expensive or
unavailable. We therefore ask when physical fields can serve as effective substitutes.

\begin{definition}[Pseudo-Sensitivity]
\label{def:pseudo_sensitivity}
Let $X:\Omega \rightarrow \R$ be a scalar physical field.
We say that $X$ is an $\epsilon$-approximate \textbf{pseudo-sensitivity} for $S_0$ if
there exists a monotone function $h:\R \rightarrow \R$ such that
\begin{equation}
\| S_0 - h(X) \|_{L^2(\Omega)}^2 \le \epsilon.
\label{eq:pseudo_sensitivity}
\end{equation}
For vector-valued fields, $X$ is first reduced to a scalar invariant (e.g.,
$\|\bm{v}\|^2$).
\end{definition}

\yunjia{Here, we provide several candidates of the pseudo-sensitivity for both fluid and structural problems.}
\begin{proposition}[Hierarchy of Physical Fields]
\label{prop:field_hierarchy}
~\\[-1em]
\begin{enumerate}
    \item \textbf{Linear elasticity (compliance):}
    Strain Energy Density is a $0$-approximate pseudo-sensitivity:
    $S_0 \propto -\mathrm{SED}$. The derivation is detailed in Appendix \ref{proof:structrual_pseudo}.
    \item \textbf{Navier--Stokes energy dissipation:}
    Under the frozen turbulence assumption and high-Reynolds-number asymptotics,
    $S_0 \propto -\|\bm{v}\|^2$. The full derivation is detailed in Appendix \ref{proof:cfd_pseudo}.
    \item \textbf{Pressure and displacement:}
    In general, these are not pseudo-sensitivities, as their relationship to $S_0$
    involves nonlocal operators and tensorial contractions.
\end{enumerate}
\end{proposition}

\paragraph{Consequences for Generalization.}
If $X$ is an $\epsilon$-approximate pseudo-sensitivity, then
\begin{equation}
H(\rho^\star \mid X)
=
H(\rho^\star \mid S_0) + \mathcal{O}(\sqrt{\epsilon}),
\label{eq:entropy_gap}
\end{equation}
and by Fano’s inequality, the minimum achievable prediction error increases smoothly
with $\epsilon$.\\
This explains the success of TopoDiff~\cite{maze2022topodiff} (which used SED), the limited generalization of methods conditioned on physical parameters~\cite{nobari2024nito}, and predicts failure for arbitrary physical fields. For CFD, the pseudo-sensitivity relationship emerges from high-Reynolds asymptotics combined with frozen turbulence (Appendix~\ref{proof:cfd_pseudo}).

\section{Method: Sensitivity-Conditioned Bernoulli Flow Matching}

We generate binary topologies using Bernoulli flow matching~\cite{gat2024bernoulli} conditioned on sensitivity fields. Unlike continuous diffusion, this formulation operates directly in Bernoulli parameter space, avoiding instabilities near the simplex boundary that arise with naive velocity regression.

\subsection{Bernoulli Probability Path}

Let $x_1 \in \{0,1\}^N$ denote the target topology. We define an uninformative prior $\alpha_{\mathrm{prior}}=0.5$ and a time-indexed family of Bernoulli marginals $\alpha_t \in [0,1]^N$ with $x_t \sim \mathrm{Bernoulli}(\alpha_t)$. Following~\cite{gat2024bernoulli}, we use a linear probability path
\begin{equation}
\alpha_t(x_1) = (1-t)\cdot 0.5 + t \cdot x_1, \qquad t\in[0,1],
\label{eq:bernoulli_path}
\end{equation}
interpolating from the uninformative prior to the target binary topology.

\subsection{Training Objective}

Rather than regressing a probability-space vector field, we train a conditional denoiser $f_\theta(x_t, t, S_0)$ that predicts the clean topology given a noisy sample $x_t$, time $t$, and conditioning signal $S_0$. The network outputs per-cell probabilities $\hat{\rho} = p_\theta(x_1{=}1 \mid x_t, t, S_0)$ and is trained with binary cross-entropy:
\begin{equation}
\mathcal{L}(\theta)=\mathbb{E}_{t,x_1,x_t}\Big[\mathrm{BCE}\big(f_\theta(x_t,t,S_0),x_1\big)\Big].
\end{equation}
We use 50 sampling steps with linear scheduling. At inference, we replace the final stochastic step with a deterministic greedy projection to produce clean, simulation-ready topologies (Appendix~\ref{sec:greedy_sampling}), and optionally enforce volume fraction constraints via confidence-based progressive pruning (Appendix~\ref{sec:volume_constraint}). Full training hyperparameters and sampling details are provided in Appendix~\ref{app:training}.

\subsection{Architecture}

\begin{table*}[t]
\centering
\footnotesize
\setlength{\tabcolsep}{4.5pt}
\caption{Simulation-level performance comparison across ID and OOD test sets. All learned models are conditioned on sensitivities. $\mathrm{\mathbf{Err}}_{\Delta p}$: Mean Relative Pressure Drop Error (\%) $\pm$ Std. \textbf{Med.}: Median Error (\%). \textbf{Acc.}: Accuracy (\% within 10\% error). STAR-CCM+ represents one optimization iteration. Best learned model results are \textbf{bolded}.}
\label{tab:combined_results}
\begin{tabular}{l ccc c ccc c ccc}
\toprule
 & \multicolumn{3}{c}{\textbf{ID Test}} & & \multicolumn{3}{c}{\textbf{OOD-Medium (2 Outlets)}} & & \multicolumn{3}{c}{\textbf{OOD-Hard (3 Outlets)}} \\
\cmidrule(lr){2-4} \cmidrule(lr){6-8} \cmidrule(lr){10-12}
\textbf{Model} & \textbf{$\mathrm{\mathbf{Err}}_{\Delta p} \pm$ Std} & \textbf{Med.} & \textbf{Acc.} & & \textbf{$\mathrm{\mathbf{Err}}_{\Delta p} \pm$ Std} & \textbf{Med.} & \textbf{Acc.} & & \textbf{$\mathrm{\mathbf{Err}}_{\Delta p} \pm$ Std} & \textbf{Med.} & \textbf{Acc.} \\
\midrule
STAR-CCM+ (1 iter) & $14.39 \pm 21.49$ & $10.28$ & $49.0$ & & $21.94 \pm 27.58$ & $15.86$ & $41.0$ & & $26.51 \pm 32.33$ & $18.47$ & $40.0$ \\
\midrule
DiT & $2.59 \pm 4.21$ & $\mathbf{1.45}$ & $96.4$ & & $8.58 \pm 14.58$ & $3.66$ & $\mathbf{79.4}$ & & $9.06 \pm 12.66$ & $4.85$ & $70.6$ \\
UDiT & $\mathbf{2.37 \pm 4.26}$ & $1.65$ & $\mathbf{97.2}$ & & $8.62 \pm 15.77$ & $4.84$ & $75.4$ & & $13.03 \pm 21.71$ & $7.00$ & $62.7$ \\
PDE-T & $2.47 \pm 4.82$ & $1.82$ & $96.8$ & & $7.50 \pm 14.23$ & $3.78$ & $75.5$ & & $9.33 \pm 17.68$ & $4.35$ & $68.6$ \\
\textbf{Ours} & $3.68 \pm 6.71$ & $2.63$ & $92.6$ & & $\mathbf{6.38 \pm 12.64}$ & $\mathbf{3.54}$ & $72.3$ & & $\mathbf{6.12 \pm 13.77}$ & $\mathbf{2.70}$ & $\mathbf{74.5}$ \\
\bottomrule
\end{tabular}
\end{table*}

Figure~\ref{fig:arch} illustrates the network architecture. We use a hierarchical vision transformer with an encoder-decoder structure: two encoder stages, a bottleneck, and two decoder stages connected by skip connections. The input consists of two channels: The conditioning field (sensitivity or pseudo-sensitivity) and the noisy Bernoulli state $x_t$which are embedded via $4{\times}4$ patches.

Each stage contains a cross-attention block followed by self-attention blocks, all using Neighborhood Attention (NATTEN) \cite{hassani2023neighborhood}. In cross-attention blocks, queries derive from the main feature stream while keys and values come from the conditioning signal, allowing the sensitivity field to directly modulate feature updates at every resolution. Timestep information is injected via adaptive layer normalization (AdaLN) \cite{DBLP:journals/corr/GoyalDGNWKTJH17,dhariwal2021diffusion}, which modulates the scale and shift of each block's activations. The decoder outputs per-cell logits that are unpatchified to the original resolution.

\subsection{Deployment Enhancements}

Beyond the generative model itself, we develop practical extensions for deployment: a greedy terminal sampling step that eliminates salt-and-pepper artifacts for simulation-ready topologies (Appendix~\ref{sec:greedy_sampling}), and a confidence-based progressive pruning strategy for volume fraction control during sampling (Appendix~\ref{sec:volume_constraint}).

\section{Experiments}

We evaluate our sensitivity-conditioned Bernoulli Flow Matching framework on turbulent CFD topology optimization and structural compliance problems. Our experiments are designed to: (i) empirically validate the information-theoretic ordering predicted by Theorem~\ref{thm:sensitivity_dpi}, (ii) demonstrate the practical benefits of pseudo-sensitivity conditioning for OOD generalization, and (iii) compare against state-of-the-art baselines.

\subsection{CFD Topology Optimization}

\paragraph{Dataset.}
\mohammadd{
We introduce a CFD-TO dataset of 2D turbulent channel flows (steady RANS, $k$--$\varepsilon$, STAR-CCM+~\cite{starccm}) minimizing energy dissipation: it consists of 10k training samples (single-outlet), 1k \textit{ID} test samples, 500 \textit{OOD-medium} samples (2 outlets), and 500 \textit{OOD-hard} samples (3 outlets); each with varied density and viscosity. Refer to Appendix \ref{app:cfd_dataset} for details.
}

\paragraph{Baselines.}
We compare against UDiT \cite{udit}, DiT \cite{dit}, and PDE-Transformer (PDE-T) \cite{DBLP:conf/icml/Holzschuh0KT25}, all conditioned on sensitivity fields and all using the small variation of the respective models. We also include STAR-CCM+ optimization after one iteration as a classical baseline.

\subsubsection{Validating the \nils{Theory}} 

To test Theorem~\ref{thm:sensitivity_dpi} and Proposition~\ref{prop:field_hierarchy}, we evaluate models conditioned on different signals using cross-entropy as a proxy for conditional entropy (see Appendix~\ref{app:bce_proxy}). Figure~\ref{fig:entropy_threeoutlet} shows results on the OOD-hard (three-outlet) test set.
\begin{figure}[h]
\centering
\includegraphics[width=\columnwidth]{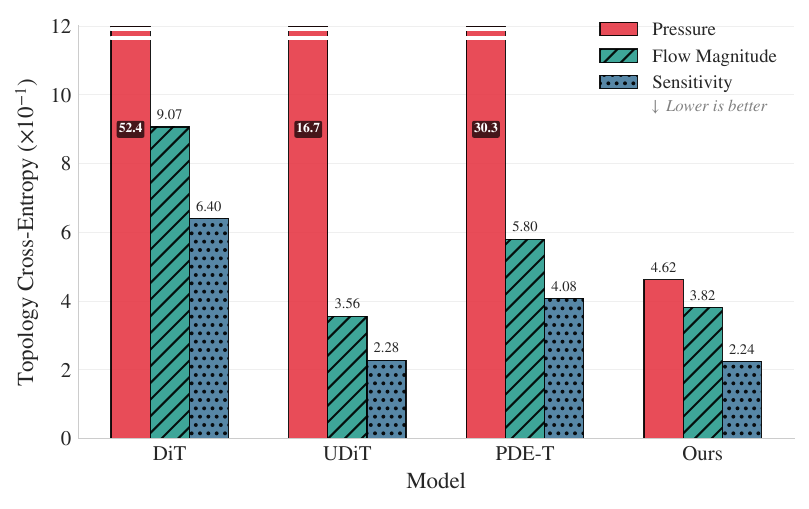}
\caption{OOD generalization on three-outlet configurations. Models conditioned on sensitivity or pseudo-sensitivity maintain low cross-entropy under distribution shift, while pressure-based conditioning degrades.}\vspace{-20pt}
\label{fig:entropy_threeoutlet}
\end{figure}
The empirical ordering is consistent with our theoretical predictions: sensitivity conditioning achieves the lowest cross-entropy, followed by pseudo-sensitivity (velocity magnitude squared), with non-pseudo fields (pressure) exhibiting the highest uncertainty. This validates the Markov chain abstraction and confirms that sensitivity proximity governs predictive performance.

\subsubsection{Model Comparison on  Simulation-Level Performance}
Table~\ref{tab:combined_results} compares all models trained with sensitivity conditioning on simulation-level metrics\mohammadd{, reporting relative error in pressure drop (see Appendix~\ref{app:metric_defs})}. We report mean relative error, median error, and accuracy (fraction of predictions within 10\% of \nils{target}).
All learned models conditioned on sensitivities outperform STAR-CCM+ after one iteration, demonstrating successful generalization to OOD settings. While UDiT and PDE-T achieve slightly better ID performance, our Bernoulli flow model exhibits the best OOD generalization, particularly on the OOD-hard (three-outlet) setting where it achieves 74.5\% accuracy compared to 68.6\% for PDE-T and 62.7\% for UDiT.

\subsubsection{Non-Pseudo Physical Fields Fail OOD}

Table~\ref{tab:conditioning_cfd} compares conditioning modalities to validate Proposition~\ref{prop:field_hierarchy}. While pressure-based conditioning achieves reasonable ID performance, it degrades substantially under OOD shifts, with median error increasing from 3.34\% to 11.51\%. In contrast, sensitivity and pseudo-sensitivity conditioning maintain stable performance across distribution shifts.
\begin{table}[bt]
\centering
\caption{Conditioning modality comparison (CFD). Non-pseudo fields (pressure) generalize ID but fail OOD. $\mathrm{\mathbf{Err}}_{\Delta p}$: Mean Relative Pressure Drop Error (\%) $\pm$ Std. Med.: Median Error (\%).}
\label{tab:conditioning_cfd}
\resizebox{\linewidth}{!}{%
\begin{tabular}{l|cc|cc}
\toprule
\multirow{2}{*}{\textbf{Modality}} & \multicolumn{2}{c|}{\textbf{In-Distribution}} & \multicolumn{2}{c}{\textbf{Out-of-Distribution}} \\
 & $\mathrm{Err}_{\Delta p} \pm$ Std & Med. & $\mathrm{Err}_{\Delta p} \pm$ Std & Med. \\ \midrule
Sensitivity & $4.16 \pm 9.20$ & $2.47$ & $5.92 \pm 15.55$ & $2.23$ \\
Pseudo-sens. & $2.69 \pm 7.09$ & $1.99$ & $1.15 \pm 24.26$ & $1.94$ \\
Pressure & $6.22 \pm 13.8$ & $3.34$ & $12.26 \pm 48.30$ & $11.51$ \\ \bottomrule
\end{tabular}%
}
\end{table}

\begin{mohammad}
\subsubsection{Qualitative Results}

Figure~\ref{fig:qualitative_cfd} compares conditioning signals on OOD three-outlet cases. Row (b) highlights the key regime shift: sensitivity and pseudo-sensitivity conditioning recover separated channel topologies that connect distinct outlets, a structure absent from the single-outlet training distribution (which only exhibits single connected flow paths). Across rows (a) and (c), sensitivity conditioning aligns most closely with the ground-truth geometry, while pseudo-sensitivity tends to smooth or slightly deform channel boundaries. Pressure conditioning fails in all cases, producing layouts that do not respect inlet/outlet placement and thus violate boundary conditions.

\begin{figure}[tb]
\centering
\includegraphics[width=\columnwidth]{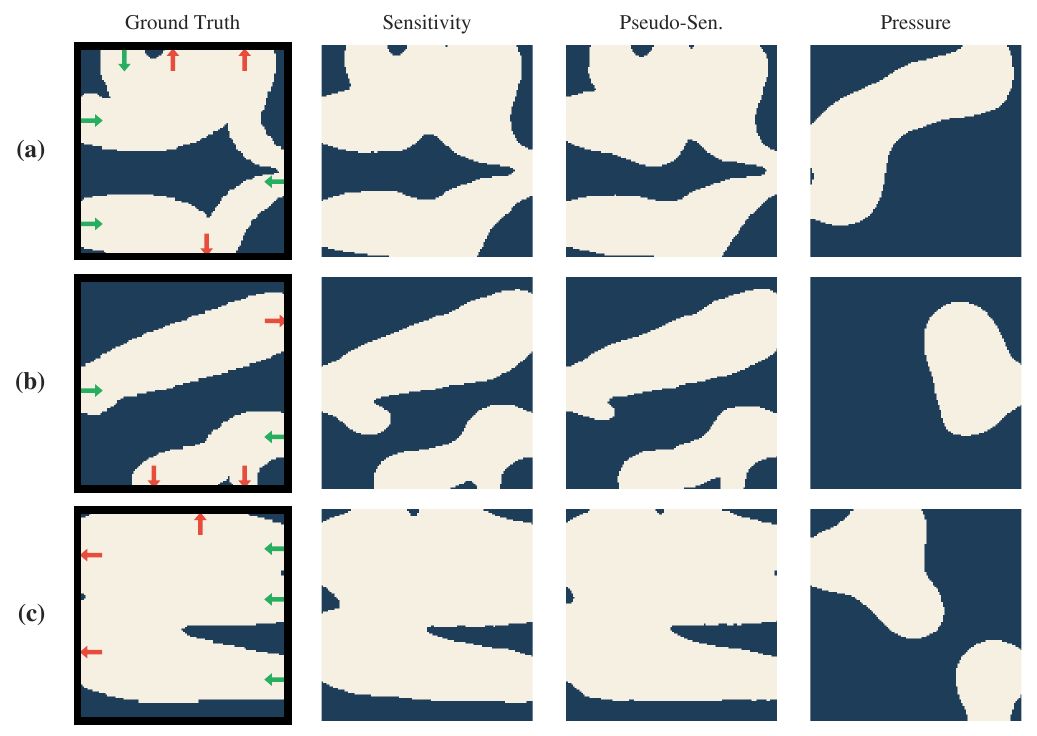}
\caption{CFD Topology Optimization on three-outlet OOD test case qualitative comparison. Sensitivity best matches the ground truth; pseudo-sensitivity generalizes but with mild geometric smoothing. Row (b) shows separated channels that do not appear in the single-outlet training set. Pressure conditioning ignores outlet constraints.\textcolor{green!50!black!50}{Green}: inflow. \textcolor{red}{Red}: outflow.}
\label{fig:qualitative_cfd}
\end{figure}
\end{mohammad}

\subsubsection{Conditioning Dominates Architecture}

To disentangle the effect of conditioning from the generative architecture, we fix one factor and vary the other on simulation-level OOD-Hard metrics. Table~\ref{tab:disentangle} shows that a weaker model with informative conditioning consistently outperforms a stronger model with an information-poor signal. UDiT conditioned on pseudo-sensitivity achieves $7.0\%$ median pressure-drop error, while our model with pressure conditioning reaches $100\%$, an order of magnitude worse despite the stronger architecture. This confirms that the conditioning signal, not the model capacity, is the primary driver of OOD performance.

\begin{table}[bt]
\centering
\caption{Conditioning vs.\ architecture on simulation-level CFD metrics (OOD-Hard, 3 outlets). Median relative pressure-drop error (\%). Rows compare the same architecture under different conditioning; columns compare different architectures under the same conditioning.}
\label{tab:disentangle}
\begin{tabular}{l@{\hskip 12pt}c@{\hskip 12pt}c}
\toprule
 & \multicolumn{2}{c}{\textbf{Architecture}} \\
\cmidrule(lr){2-3}
\textbf{Conditioning} & UDiT & \textbf{Ours (BFM)} \\
\midrule
Pseudo-sens.\ ($\|\mathbf{v}\|^2$) & $7.0$ & $\mathbf{2.7}$ \\
Non-pseudo (pressure) & $68.0$ & $100.0$ \\
\bottomrule
\end{tabular}
\end{table}

\subsection{Structural Compliance Optimization}

\paragraph{Dataset.}
We use the established structural topology optimization benchmark from TopoDiffusion~\cite{maze2022topodiff} and NITO~\cite{nobari2024nito}, consisting of 2D cantilever and bridge-like structures under linear elasticity with compliance minimization. The dataset includes diverse load configurations with OOD test cases containing unseen loading patterns.

\paragraph{Baselines.}
\nils{We compare against the corresponding models, TopoDiffusion, a diffusion model conditioned on strain/stress energy density, and NITO, a deterministic implicit network conditioned on physical parameters.} 

\subsubsection{Near-Zero Gap Between SED and Sensitivity}

For structural compliance, Proposition~\ref{prop:field_hierarchy} establishes that strain energy density (SED) is a $0$-approximate pseudo-sensitivity. Figure~\ref{fig:entropy_conditioning} confirms this prediction: the cross-entropy gap between sensitivity and SED conditioning is negligible, while conditioning on non-pseudo fields (displacement) or upstream physical parameters yields substantially higher uncertainty.

\begin{figure}[bt]
\centering
\includegraphics[width=\columnwidth]{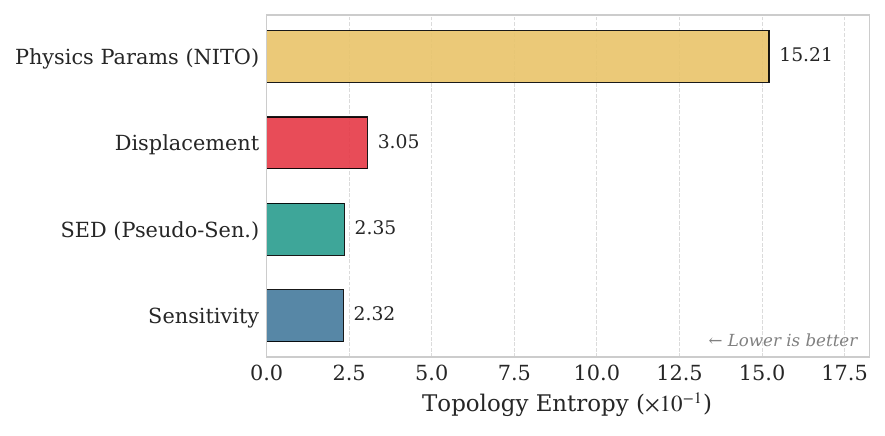}
\caption{Cross-entropy across conditioning signals for structural problems. Sensitivity and SED (pseudo-sensitivity) achieve nearly identical CE, while displacement and physical parameters exhibit higher uncertainty, consistent with Proposition~\ref{prop:field_hierarchy}. CE computed with logit temperature $T{=}15$.}
\label{fig:entropy_conditioning}
\end{figure}

\begin{table}[bt]
\centering
\caption{Cross-entropy and compliance error comparison for structural problems. SED (pseudo-sensitivity) matches true sensitivity, validating Proposition~\ref{prop:field_hierarchy}. CE is computed with logit temperature $T{=}15$. \textbf{CE}: Cross Entropy. $\mathrm{\mathbf{Err}}_{C}$ \textbf{Med.}: Median Compliance Error.  }
\label{tab:structural_ce}
\begin{tabular}{lcc}
\toprule
\textbf{Conditioning} & \textbf{CE} ($\times 10^{-1}$) & \textbf{$\mathrm{\mathbf{Err}}_{C}$ Med.(\%)} \\
\midrule
Sensitivity & $2.32$ & $0.54$ \\
SED (Pseudo-sens.) & $2.35$ & $0.53$ \\
Displacement & $3.05$ & $1.30$ \\
Physical Params & $15.21$ & $14.0$ \\
\bottomrule
\end{tabular}
\end{table}

\mohammadd{
Table~\ref{tab:structural_ce} quantifies this ordering using both cross-entropy and median compliance error (see Appendix~\ref{app:metric_defs} for metric definitions). The near-identical performance for sensitivity and SED ($2.32\&0.54$ vs $2.35\&0.53$) confirm that SED preserves essentially all information about the optimal topology,} consistent with the $0$-approximate pseudo-sensitivity derivation in Appendix~\ref{proof:structrual_pseudo}. Physical parameter conditioning exhibits the highest cross-entropy, explaining the limited OOD generalization of methods like NITO that rely solely on boundary conditions and loads.
\begin{mohammad}
    
\subsubsection{Deterministic Models Can Generalize}
\begin{table}[bt]
\centering
\caption{Comparison of generative and deterministic surrogate models under different conditioning signals. Sensitivity conditioning enables strong generalization even for deterministic models, while conditioning on physical parameters alone leads to substantially higher compliance error.}
\label{tab:determintistic_ce}
\begin{tabular}{lcc}
\toprule
\textbf{Conditioning} & Model Type& \textbf{$\mathrm{\mathbf{Err}}_{C}$ Med.(\%)} \\
\midrule
Sensitivity & Generative & $0.54$ \\
Sensitivity & Deterministic & $2.27$ \\
Physical Params & Deterministic & $14.0$ \\
\bottomrule
\end{tabular}
\end{table}
A common speculation regarding the generalizability of surrogate topology optimization models is that it primarily stems from their generative nature. Table~\ref{tab:determintistic_ce} demonstrates that non-generative (deterministic) models can also generalize effectively when provided with informative conditioning. Notably, while the generative sensitivity-conditioned model achieves a lower median compliance error than its deterministic counterpart ($0.54\%$ vs.\ $2.27\%$), the deterministic sensitivity-conditioned model still outperforms the physics-parameter–conditioned deterministic model by a wide margin ($2.27\%$ vs.\ $14.0\%$). This gap underscores that informative conditioning is the dominant factor governing generalization performance even in the absence of a generative modeling framework. This supports Theorem~\ref{thm:sensitivity_dpi}, which shows that the choice of conditioning signal is more critical to generalization than whether the model is generative.
\end{mohammad}

\subsubsection{Comparison with State-of-the-Art}

\begin{table}[bt]
\centering
\small
\caption{Performance comparison on structural TO. FS: further SIMP steps after model output.}
\label{tab:structural_sota}
\begin{tabular}{l c c cc}
\toprule
\textbf{Model} & \textbf{Params (M)} & \textbf{FS} & \textbf{Mean} & \textbf{Med.} \\ 
\midrule
\textit{OOD testing} & & & \multicolumn{2}{c}{$\mathrm{\mathbf{Err}}_{C}$\% $\downarrow$} \\ 
\cmidrule(lr){4-5}
\textbf{Ours} & 34 & -- & \textbf{5.73} & \textbf{0.53} \\
TopoDiff & 121 & -- & 8.57 & 1.14 \\
TopoDiff w/ G & 239 & -- & 7.79 & 1.26 \\ 
NITO & \textbf{22} & 5 & 9.33 & 2.37 \\ 
\bottomrule
\end{tabular}
\end{table}

We compare our method against TopoDiffusion and NITO on OOD structural problems. 
Our sensitivity-conditioned Bernoulli flow model achieves state-of-the-art performance with 34M parameters, outperforming TopoDiffusion (121M parameters) by 33\% in mean compliance error ($\mathrm{\mathbf{Err}}_{C}$) and NITO (22M parameters) by 39\%, \nils{as shown in Table~\ref{tab:structural_sota}}. 
Notably, NITO requires 5 additional SIMP iterations after inference, while our method produces final topologies directly.
These results demonstrate that the combination of optimal conditioning (sensitivities) and appropriate generative modeling (Bernoulli flows for binary topology) yields substantial improvements over prior work in both accuracy and efficiency.

\begin{mohammad}
\subsubsection{Qualitative Results (Structural)}

Figure~\ref{fig:qualitative_structural} shows OOD structural examples.
As expected from Proposition~\ref{prop:field_hierarchy}, sensitivity and SED conditioning produce nearly indistinguishable topologies. Notably, displacement conditioning yields qualitatively plausible structures that capture the principal load paths. This suggests that displacement, while not a pseudo-sensitivity in the strict sense, retains sufficient structural information to approximate the sensitivity landscape (at the cost performance degradation). The progressive degradation from sensitivity through SED to displacement is consistent with Proposition~\ref{prop:field_hierarchy}.

\begin{figure}[tb]
\centering
\includegraphics[width=\columnwidth]{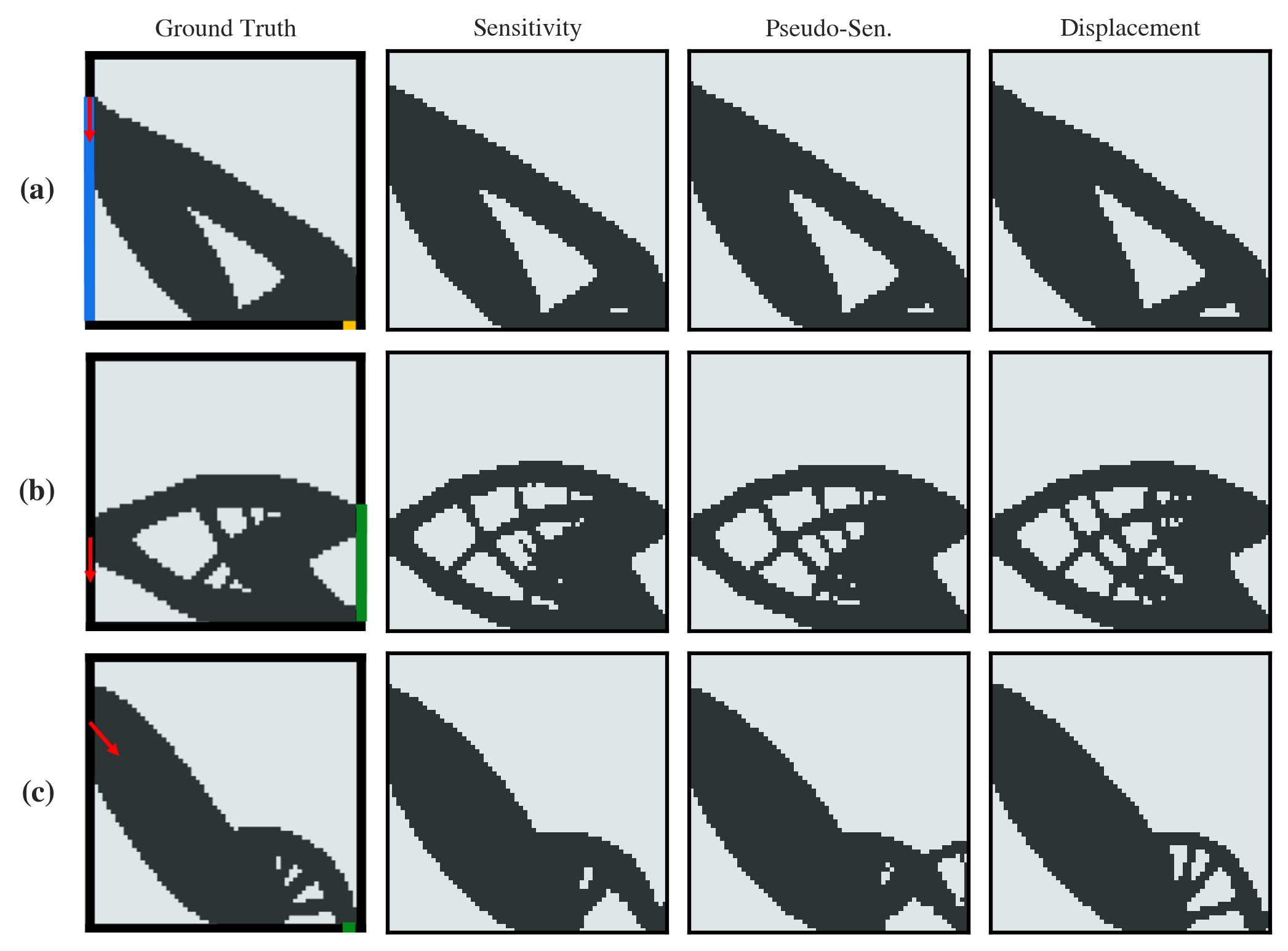}
\caption{OOD structural qualitative comparison. Sensitivity and SED (pseudo-sensitivity) are visually indistinguishable, while displacement conditioning produces less reliable topologies. \textcolor{blue!100!black!70}{Blue}: constrained in y-axis. \textcolor{yellow!80!black!75}{Yellow}: constrained in x-axis. \textcolor{green!50!black!50}{Green}: constrained in both axes. \textcolor{red}{Red}: applied force.}
\label{fig:qualitative_structural}
\end{figure}
\end{mohammad}

\subsubsection{Computational Cost}

Table~\ref{tab:speed} compares end-to-end wall-clock time. Our method requires a single primal PDE solve for conditioning ({\raise.17ex\hbox{$\scriptstyle\sim$}}3.2\,s for CFD) followed by near-instant neural inference (0.65\,ms per sample at batch size 256). This is {\raise.17ex\hbox{$\scriptstyle\sim$}}2$\times$ faster than NITO with 5 SIMP steps and orders of magnitude faster than classical iterative optimization in STAR-CCM+, which requires {\raise.17ex\hbox{$\scriptstyle\sim$}}114\,s for convergence.
Pseudo-sensitivity conditioning avoids the additional adjoint solve ({\raise.17ex\hbox{$\scriptstyle\sim$}}5.3\,s), reducing the total physics cost by $>$2.5$\times$ compared to exact sensitivities while retaining strong OOD performance.

\begin{table}[bt]
\centering
\small
\caption{Computational cost comparison (CFD). Ours replaces iterative optimization with a single primal solve and fast inference. Pseudo-sensitivities avoid the costlier adjoint computation.}
\label{tab:speed}
\begin{tabular}{lc}
\toprule
\textbf{Component} & \textbf{Time / sample} \\
\midrule
STAR-CCM+ full optimization & $113.94$ s \\
\quad Per SIMP iteration & $\sim$4.32 s \\
\midrule
NITO + 5 SIMP steps & $82.5$ ms \\
\midrule
Primal PDE solve (CFD) & $\sim$3.2 s \\
Adjoint solve (additional) & $\sim$5.3 s \\
\midrule
\textbf{Ours} (inference, batch 256) & $0.65$ ms \\
\textbf{Ours} (primal + inference) & $\sim$3.2 s \\
\textbf{Ours} (exact sens. + inference) & $\sim$8.5 s \\
\bottomrule
\end{tabular}
\end{table}

\section{Conclusions and Discussion}

We have presented a principled framework for generalization in topology optimization based on sensitivity-conditioned generative modeling. By combining theoretical insights, pseudo-sensitivity derivations, and discrete Bernoulli flow matching, we enable efficient and generalizable topology optimization and introduce a challenging CFD benchmark.

\paragraph{Scope and limitations.}
The Markov abstraction $\Theta \to X \to S_0 \to \rho^*$ is used as a presentation tool for a clean application of the DPI; the information ordering $H(\rho^* \mid S) \le H(\rho^* \mid X)$ holds more generally at each optimization iteration, since sensitivities are objective-aligned functionals of the state. The abstraction may weaken in settings with strong history dependence (e.g., multi-objective or path-dependent constraints), though empirically the initial sensitivity already captures the dominant topology structure (Appendix~\ref{app:auc_sufficiency}).

Our CFD pseudo-sensitivity derivation employs the frozen turbulence assumption, which is standard in steady RANS adjoint optimization~\cite{giles2000introduction,papoutsis2016continuous}. While this introduces approximation error, particularly in transitional Reynolds-number regimes, it was applied consistently in both data generation and sensitivity computation. Quantitative validation across our full dataset confirms robustness, with Pearson correlation between pseudo-sensitivity and true sensitivity exceeding $r = 0.95$ on average (Appendix~\ref{app:pseudo_validation}).

A natural limitation is that our method requires a physics solve for the conditioning signal. Pseudo-sensitivities mitigate this by relying only on primal simulations (avoiding the costlier adjoint), and benefit directly from ongoing advances in GPU-accelerated and neural PDE solvers. The current evaluation focuses on 2D structural compliance and 2D turbulent CFD; extending to 3D domains and additional PDE families is important future work.

\paragraph{Outlook.}
Our results indicate that generalization in topology optimization arises from sensitivity conditioning rather than the choice of generative mechanism. 
We expect this fundamental insight to be crucial for a wide variety of future, learning-based applications ranging from engineering design to medical scenarios.

\section*{Impact Statement}
\nils{This research aims to advance machine learning techniques for topology optimization. More efficient and accurate optimization capabilities can facilitate scientific and engineering progress in areas such as transportation design, aerodynamics and for medical procedures. We do not identify any immediate or specific negative societal impacts associated with the proposed methods.}

\bibliographystyle{icml2026}
\bibliography{references_updated}


\newpage
\appendix
\onecolumn

\section*{\begin{center}\Large Appendix\end{center}} 
\vspace{0.5cm}

\section{Proofs and Derivations}

\subsection{Proof of Data Processing Inequality for Conditioning Signals}
\label{proof:dpi}

Theorem~\ref{thm:sensitivity_dpi} follows directly from the Data Processing Inequality (DPI) applied to the Markov chain structure in Assumption~\ref{assump:markov_theory}.

\begin{proof}
By Assumption~\ref{assump:markov_theory}, the random variables form the Markov chain $\Theta \rightarrow X \rightarrow S_0 \rightarrow \rho^\star$. The Data Processing Inequality states that for any Markov chain $A \rightarrow B \rightarrow C$, we have $I(A; C) \le I(B; C)$.

Applying DPI to $\Theta \rightarrow X \rightarrow \rho^\star$ yields $I(\Theta; \rho^\star) \le I(X; \rho^\star)$.

Applying DPI to $X \rightarrow S_0 \rightarrow \rho^\star$ yields $I(X; \rho^\star) \le I(S_0; \rho^\star)$.

Combining these inequalities gives the mutual information ordering in Eq.~\eqref{eq:mi_ordering}. The conditional entropy ordering in Eq.~\eqref{eq:entropy_ordering} follows from the identity $H(\rho^\star \mid Z) = H(\rho^\star) - I(\rho^\star; Z)$, which shows that higher mutual information corresponds to lower conditional entropy.

This establishes that among all conditioning signals derivable from the problem parameters $\Theta$, the sensitivity field $S_0$ provides the tightest bound on the irreducible uncertainty in predicting the optimal topology $\rho^\star$.
\end{proof}

\subsection{BCE as a Mutual-Information Proxy}
\label{app:bce_proxy}

Here $Z:=\rho^\star\in\{0,1\}^{d}$ is the (binary) converged topology, and we compare conditioning signals $C\in\{X,S_0\}$ (physical fields vs.\ initial sensitivities) via their ability to predict $Z$.

Let $q_\phi(z\mid C)$ be a probabilistic decoder and define the population cross-entropy (log-loss) $\CE_C(\phi):=\E[-\log q_\phi(Z\mid C)]$. A standard decomposition gives
\begin{equation}
\begin{split}
\CE_C(\phi)
&= H(Z\mid C)
 + \E_C\!\Big[D_{\mathrm{KL}}\!\big(p(\cdot\mid C)\,\|\,q_\phi(\cdot\mid C)\big)\Big] \\
&\ge H(Z\mid C),
\end{split}
\label{eq:ce_upper_entropy}
\end{equation}
with equality if and only if $q_\phi(\cdot\mid C)=p(\cdot\mid C)$ almost surely.

Therefore the quantity $\widehat I_\phi(Z;C):=H(Z)-\CE_C(\phi)$ is a computable \emph{lower bound} on the mutual information $I(Z;C)=H(Z)-H(Z\mid C)$: $\widehat I_\phi(Z;C)\le I(Z;C)$.

In our experiments, we evaluate $\CE_X$ and $\CE_{S_0}$ using the same decoder family and training recipe; then $\CE_{S_0}<\CE_X$ implies a strictly larger information lower bound $\widehat I_\phi(Z;S_0)>\widehat I_\phi(Z;X)$. Log-loss is also a strictly proper scoring rule, making it suitable for comparing conditional uncertainty across signals.

\subsection{Sensitivity Sufficiency at the Optimum}
\label{app:sensitivity_sufficiency}

We connect the information-theoretic framework to classical optimality conditions.

\begin{proposition}[Sensitivity Sufficiency]
\label{prop:sensitivity_sufficiency}
Under standard KKT conditions for density-based topology optimization, the optimal design $\rho^\star$ is uniquely determined (up to measure-zero singular sets) by the terminal sensitivity field $S(\rho^\star)$ and the volume multiplier $\nu$.
\end{proposition}

\paragraph{Implication.}
Conditioning on $S_0$ allows a model to implicitly learn the thresholding operation induced by $\nu$, making the sensitivity field a sufficient statistic for predicting $\rho^\star$ in practice. A detailed proof follows in Section~\ref{proof:kkt}.

\subsection{Proof of Sensitivity Sufficiency and Volume Constraints}
\label{proof:kkt}

We now show that, at the optimum, the topology $\rho^\star$ is determined (up to measure-zero singular sets) by the pair consisting of the terminal sensitivity field $S(\rho^\star)$ and the volume multiplier $\nu$.

\begin{proof}
We consider the reduced optimization problem in $\rho$, where the state $u$ has been eliminated via the PDE constraint, so that
\[
\Obj(\rho) := \Obj(u(\rho), \rho), \qquad \Res(\rho) := \Res(u(\rho), \rho),
\]
and $S(\rho)$ denotes the reduced gradient $d\Obj/d\rho$ as in Eq.~\eqref{eq:sensitivity}. The reduced Lagrangian reads
\[
L(\rho, \nu, \bm{\mu}, \bm{\eta}) = \Obj(\rho) + \nu \left( \int_\Omega \rho \, dx - \Vol_{req} \right) - \bm{\mu}^\top \rho + \bm{\eta}^\top (\rho - 1),
\]
where $\bm{\mu}, \bm{\eta} \ge 0$ collect the Lagrange multipliers for the box constraints $0 \le \rho \le 1$.

The KKT stationarity condition with respect to $\rho$ is, componentwise,
\[
\frac{\partial L}{\partial \rho_i} = S_i(\rho) + \nu - \mu_i + \eta_i = 0.
\]
For the optimal solution $\rho^\star$, the complementarity slackness for the box constraints implies:
\begin{equation}
    \rho^\star_i = 
    \begin{cases}
        1 & \text{if } S_i(\rho^\star) + \nu < 0, \\
        0 & \text{if } S_i(\rho^\star) + \nu > 0, \\
        \xi \in [0,1] & \text{if } S_i(\rho^\star) + \nu = 0 \quad (\text{singular arc}).
    \end{cases}
\end{equation}
Thus, except on measure-zero singular sets where $S_i(\rho^\star) + \nu = 0$, the optimal design $\rho^\star$ is uniquely determined by the pair $(S(\rho^\star), \nu)$. In typical penalized TO settings, these singular cases vanish in the limit, so $(S(\rho^\star),\nu)$ can be regarded as determining $\rho^\star$ in practice.

In our learning framework, $\nu$ is a global scalar derived from $S$ to satisfy the volume constraint $\Vol_{req}$. Conditioning on the field $S$ therefore allows the network to implicitly learn the thresholding operation $\nu(S)$, and $S$ acts as a sufficient statistic field for predicting $\rho^\star$ up to this scalar threshold.
\end{proof}

\subsection{Derivation of Pseudo-Sensitivity for Compliance}
\label{proof:structrual_pseudo}
\begin{proof}
Let $\Obj = \bm{f}^\top \bm{u}$ (Compliance) and $\bm{K}(\rho)\bm{u} = \bm{f}$ (Equilibrium). Under a standard SIMP interpolation, the element stiffness is $\bm{K}_e(\rho_e) = \rho_e^p \bm{K}_0$ for some penalization exponent $p > 1$. The adjoint equation is
\[
\bm{K}(\rho)\bm{\lambda} = -\bm{f},
\]
so that $\bm{\lambda} = -\bm{u}$ (self-adjoint problem). The element-wise sensitivity is
\[
S_e = \bm{\lambda}_e^\top \frac{\partial \bm{K}_e}{\partial \rho_e} \bm{u}_e = -\bm{u}_e^\top \big(p \rho_e^{p-1} \bm{K}_0\big) \bm{u}_e.
\]
At iteration 0, we take $\rho_e = 1$ everywhere, so
\[
S_e = - p \,\bm{u}_e^\top \bm{K}_0 \bm{u}_e \propto - \text{Strain Energy Density}.
\]
This shows that, at the initial design, the Strain Energy Density is a $0$-approximate Pseudo-Sensitivity, confirming item (1) of Proposition~\ref{prop:field_hierarchy}.
\end{proof}

\subsection{Derivation of Pseudo-Sensitivity for Energy Dissipation in Navier--Stokes Flow}
\label{proof:cfd_pseudo}

We consider incompressible Navier--Stokes (NS) flow with an energy dissipation objective and derive the exact sensitivity field. Under the frozen turbulence assumption and high-Reynolds number asymptotics, we show that the sensitivity reduces to a pseudo-sensitivity proportional to $-\|\vv\|^2$.

\subsubsection{Problem Setup: Navier--Stokes with Brinkman Penalization}

We consider incompressible NS flow in a design domain $\Omega \subset \R^d$ ($d=2,3$) with the material distribution described by a porosity-dependent Brinkman term $\alpha(\rho)$, where $\rho \in L^\infty(\Omega;[0,1])$ is the design variable. Under the \emph{frozen turbulence} assumption, we model the effective kinematic viscosity as constant,
\[
\nu_{\text{eff}} = \nu + \nu_t,
\]
where $\nu$ is the molecular viscosity and $\nu_t$ is the turbulent eddy viscosity, treated as independent of $\rho$ during differentiation ($\delta \nu_t = 0$).

The steady primal NS equations (in velocity--pressure form) are
\begin{equation}
\label{eq:ns_primal}
\begin{aligned}
\Res_{\vv}(\vv, p, \rho) &:= (\vv \cdot \nabla)\vv - \nu_{\text{eff}} \Delta \vv + \frac{1}{\rho_f}\nabla p + \alpha(\rho)\,\vv = \bm{0} \quad \text{in } \Omega, \\
\Res_{c}(\vv) &:= \nabla \cdot \vv = 0 \quad \text{in } \Omega,
\end{aligned}
\end{equation}
with appropriate boundary conditions (e.g., prescribed inlet velocity, zero-pressure outlet, no-slip walls). Here $\vv$ is the velocity field, $p$ the pressure, $\rho_f$ the fluid density (constant), and $\alpha(\rho)$ enforces solid regions via large drag as $\rho \to 0$.

\subsubsection{Energy Dissipation Functional}

We take as objective the total viscous and Brinkman dissipation
\begin{equation}
\label{eq:dissipation_functional}
\Obj(\vv, \rho) = \int_\Omega \left[ \nu_{\text{eff}} \, \|\nabla \vv\|^2 + \alpha(\rho) \, \|\vv\|^2 \right] \, dV,
\end{equation}
where $\|\nabla \vv\|^2 = \nabla \vv : \nabla \vv = \sum_{i,j} (\partial_j v_i)^2$.

The corresponding topology optimization problem is
\begin{equation}
\min_{\rho \in [0,1]^N} \Obj(\vv(\rho), \rho) \quad \text{s.t.} \quad \Res_{\vv}(\vv, p, \rho) = \bm{0}, \ \Res_c(\vv) = 0, \ \int_\Omega \rho \, dV \le \Vol_{\text{req}}.
\end{equation}

\subsubsection{Lagrangian and Adjoint Equations}

We form the Lagrangian
\begin{equation}
\mathcal{L}(\vv, p, \rho, \ww, q) = \Obj(\vv, \rho) + \int_\Omega \ww \cdot \Res_{\vv}(\vv, p, \rho) \, dV + \int_\Omega q \, \Res_{c}(\vv) \, dV,
\end{equation}
where $\ww$ and $q$ are the adjoint velocity and adjoint pressure.

Taking variations with respect to the state variables $(\vv, p)$ and imposing stationarity yields the adjoint system. Under the frozen turbulence assumption ($\delta \nu_{\text{eff}} = 0$), we obtain (after standard integration-by-parts manipulations):
\begin{equation}
\label{eq:ns_adjoint}
\begin{aligned}
(\nabla \vv)^\top \ww - (\vv \cdot \nabla)\ww - \nu_{\text{eff}} \Delta \ww + \alpha(\rho)\,\ww + \nabla q &= -2\nu_{\text{eff}} \Delta \vv + 2\alpha(\rho)\,\vv \quad \text{in } \Omega, \\
\nabla \cdot \ww &= 0 \quad \text{in } \Omega,
\end{aligned}
\end{equation}
with suitable adjoint boundary conditions (e.g., homogeneous Dirichlet at inlets and walls, characteristic conditions at outlets). The right-hand side arises from differentiating the dissipation functional \eqref{eq:dissipation_functional} with respect to $\vv$.

\subsubsection{Sensitivity with Respect to the Design}

The reduced gradient (sensitivity) with respect to $\rho$ is given by differentiating the Lagrangian while enforcing the adjoint equations:
\begin{equation}
S(x) := \frac{d\Obj}{d\rho}(x) = \frac{\partial \Obj}{\partial \rho}(x) - \ww(x) \cdot \frac{\partial \Res_{\vv}}{\partial \rho}(x).
\end{equation}

From \eqref{eq:ns_primal} and \eqref{eq:dissipation_functional}, we have
\begin{align}
\frac{\partial \Obj}{\partial \rho} &= \|\vv\|^2 \, \frac{d\alpha}{d\rho}, \\
\frac{\partial \Res_{\vv}}{\partial \rho} &= \frac{d\alpha}{d\rho} \, \vv.
\end{align}

Substituting into the sensitivity expression yields
\begin{equation}
S(x) = \underbrace{\|\vv(x)\|^2 \, \frac{d\alpha}{d\rho}}_{\text{explicit dissipation term}} - \ww(x) \cdot \left( \frac{d\alpha}{d\rho} \, \vv(x) \right) = \frac{d\alpha}{d\rho} \left( \|\vv\|^2 - \ww \cdot \vv \right).
\end{equation}

A common choice for the Brinkman interpolation is, for example, a RAMP-type law
\begin{equation}
\alpha(\rho) = \alpha_{\min} + (\alpha_{\max} - \alpha_{\min}) \frac{1-\rho}{1+q\rho}, \quad q > 0,
\end{equation}
which yields
\begin{equation}
\frac{d\alpha}{d\rho} = - (\alpha_{\max} - \alpha_{\min}) \, \frac{1+q}{(1+q\rho)^2} < 0.
\end{equation}
Thus we can write
\begin{equation}
S(x) = -\left| \frac{d\alpha}{d\rho} \right| \left( \|\vv(x)\|^2 - \ww(x) \cdot \vv(x) \right).
\end{equation}

\subsubsection{Low-Reynolds Number (Stokes) Limit}

At low Reynolds numbers ($\mathrm{Re} \ll 1$), inertial terms vanish from both the primal \eqref{eq:ns_primal} and adjoint \eqref{eq:ns_adjoint} systems. Both reduce to linear Brinkman--Stokes operators of the form $-\nu_{\text{eff}} \Delta \mathbf{u} + \alpha(\rho)\,\mathbf{u} + \nabla \pi = \mathbf{f}$, with $\nabla \cdot \mathbf{u} = 0$. The adjoint forcing is $\mathbf{f} = -2\nu_{\text{eff}} \Delta \vv + 2\alpha(\rho)\,\vv$, for which a particular solution is $\ww = -\vv$ (verified by direct substitution). Therefore, in the Stokes regime:
\[
\ww = -\vv \quad \Longrightarrow \quad S(x) = -2\left|\frac{d\alpha}{d\rho}\right| \|\vv(x)\|^2 \propto -\|\vv(x)\|^2.
\]

\subsubsection{High-Reynolds Number Asymptotics}

At high Reynolds numbers, the global flow is convection-dominated, but in optimized designs the characteristic channel width $h$ becomes small and the \emph{local} Reynolds number $\mathrm{Re}_h = U h / \nu_{\text{eff}}$ can be much less than one inside thin flow channels. Under the frozen turbulence assumption ($\delta \nu_t = 0$), the adjoint equation in these micro-channels reduces to the same Stokes-like balance as above. A scaling argument shows that the adjoint velocity again satisfies $\ww \approx -\vv$, yielding:
\[
S(x) \approx -2\left|\frac{d\alpha}{d\rho}\right| \|\vv(x)\|^2 \propto -\|\vv(x)\|^2.
\]

\subsubsection{Robustness Across Reynolds Regimes}

The proportionality $S \propto -\|\vv\|^2$ thus holds in both asymptotic limits (low and high $\mathrm{Re}$), with deviations confined to the transitional regime where neither limit fully applies. In the notation of Definition~\ref{def:pseudo_sensitivity}, this corresponds to $X(x) = \|\vv(x)\|^2$ and a linear monotone function $h(s) = -Cs$ for some $C > 0$, giving $S_0(x) = h(X(x))$ with $\epsilon = 0$ in the asymptotic regimes.

\subsubsection{Validity of the Frozen Turbulence Assumption}

The frozen turbulence assumption ($\delta \nu_t = 0$) becomes increasingly valid at high Reynolds numbers because:
\begin{itemize}
    \item Turbulent fluctuations evolve on much faster time scales than design variations, so $\nu_t$ can be regarded as quasi-steady with respect to $\rho$.
    \item Optimized topologies tend to adjust the \emph{mean} flow structure more strongly than the small-scale turbulence, so the leading-order sensitivity of dissipation is captured by variations in $\vv$ and $\alpha(\rho)$.
    \item In narrow channels, the effective viscosity is dominated by $\nu_t$, and the flow tends towards locally Stokes-like behavior despite a large global Reynolds number, reinforcing the approximation $\ww \approx -\vv$.
\end{itemize}

Under these conditions, the approximation $S \propto -\|\vv\|^2$ holds to leading order, justifying the classification of the velocity magnitude squared as a $0$-approximate Pseudo-Sensitivity for the energy dissipation objective in NS-based topology optimization, as stated in item (2) of Proposition~\ref{prop:field_hierarchy}. Quantitative validation across our full CFD dataset (Appendix~\ref{app:pseudo_validation}) confirms that the Pearson correlation between pseudo-sensitivity and true sensitivity exceeds $r = 0.95$ on average, with degradation confined to a narrow transitional Reynolds-number band.

\begin{mohammad}
\section{Training and Evaluation Details}
\label{app:training}

\subsection{Hardware and Software Configuration}

All experiments were conducted on a single NVIDIA L40S GPU with 46\,GB of memory. The training server is equipped with an AMD EPYC 7R13 processor (16 cores, 32 threads) and 248\,GB of system RAM. We used PyTorch 2.9.1 with CUDA 12.8 and mixed-precision training (bf16) to accelerate computation.

\subsection{CFD Topology Optimization}

For the fluid dynamics problem, we trained four model variants on the sensitivity field prediction task with a resolution of $128 \times 128$. We initialized training with 50 epochs for all models; for models that did not converge, we progressively increased the training budget by 50 epochs until satisfactory performance was achieved. Table~\ref{tab:cfd_training} summarizes the final training hyperparameters.

\begin{table}[h]
\centering
\caption{Training hyperparameters for CFD topology optimization experiments.}
\label{tab:cfd_training}
\begin{tabular}{lcccccc}
\toprule
Model & Batch Size & Grad Accum & Eff. Batch & Learning Rate & Max Steps & Max Epochs \\
\midrule
Ours  & 256 & 1 & 256 & $5 \times 10^{-4}$ & 5,000 & 50 \\
PDE-T & 256 & 1 & 256 & $5 \times 10^{-4}$ & 10,000 & 100 \\
UDiT  & 128 & 2 & 256 & $5 \times 10^{-4}$ & 10,000 & 100 \\
DiT   & 128 & 2 & 256 & $3 \times 10^{-4}$ & 15,000 & 150 \\
\bottomrule
\end{tabular}
\end{table}

Due to memory constraints, UDiT and DiT were trained with a batch size of 128 with gradient accumulation over 2 steps to simulate an effective batch size of 256, matching the other methods. All models employ exponential moving average (EMA) with a decay rate of 0.999 for evaluation. We use linear sampling with 50 denoising steps for our method. During evaluation on the pressure drop relative error metric, we filter out samples where the predicted design deviates more than 50\% from the initial geometry to ensure fair comparison across all models.

\subsection{Structural Topology Optimization}

For the compliance minimization problem, Table~\ref{tab:structural_training} summarizes the training hyperparameters.

\begin{table}[h]
\centering
\caption{Training hyperparameters for structural topology optimization experiments.}
\label{tab:structural_training}
\begin{tabular}{lccccc}
\toprule
Model & Batch Size & Learning Rate & Max Steps & Max Epochs \\
\midrule
Ours  & 256 & $5 \times 10^{-4}$ & 27,600 & 150 \\
\bottomrule
\end{tabular}
\end{table}

We use a window size of 33 for neighborhood attention and 50 sampling steps with linear scheduling. Following the NITO evaluation protocol, we filter out samples with compliance errors exceeding 1000\% when computing the relative error metric.

\subsection{Common Settings}

All models are trained with the AdamW optimizer. We employ adaptive gradient clipping with EMA-based norm estimation (coefficients 0.9 and 0.99) to stabilize training. Data normalization follows a mean-std scheme for both inputs and physical constants.
\subsection{Metric Definitions}
\label{app:metric_defs}

\paragraph{CFD (pressure-drop relative error).}
We evaluate CFD solutions by relative error in pressure drop,
\[
\mathrm{Err}_{\Delta p}
=
\frac{\Delta p_{\mathrm{pred}}-\Delta p_{\mathrm{ref}}}{\Delta p_{\mathrm{ref}}}.
\]
The pressure drop $\Delta p$ is computed as the mass‑flow‑weighted total pressure difference between inlet and outlet boundaries:
\[
\Delta p
=
\left[
\frac{\sum_{f\in \Gamma_{\mathrm{in}}} |\dot{m}_f|\, p_{t,f}}
     {\sum_{f\in \Gamma_{\mathrm{in}}} |\dot{m}_f|}
\right]
-
\left[
\frac{\sum_{f\in \Gamma_{\mathrm{out}}} |\dot{m}_f|\, p_{t,f}}
     {\sum_{f\in \Gamma_{\mathrm{out}}} |\dot{m}_f|}
\right],
\]
where $\Gamma_{\mathrm{in}}$ and $\Gamma_{\mathrm{out}}$ are inlet and outlet boundary faces, $\dot{m}_f$ is the signed mass flow rate through face $f$, and $p_{t,f}$ is the absolute total pressure at face $f$.

\paragraph{Structural (compliance relative error).}
We evaluate structural solutions by relative compliance error,
\[
\mathrm{Err}_{C}
=
\frac{C_{\mathrm{pred}}-C_{\mathrm{ref}}}{C_{\mathrm{ref}}},
\qquad
C=\bm{f}^{\top}\bm{u},
\]
where $\bm{f}$ is the applied nodal load vector and $\bm{u}$ is the displacement vector obtained from the linear elasticity solve. $C_{\mathrm{ref}}$ corresponds to the compliance of the reference (ground‑truth) topology and $C_{\mathrm{pred}}$ to the compliance of the predicted topology, both evaluated under identical loads and boundary conditions.

\section{Additional Results}
\label{app_results}

\begin{table*}[b]
\centering
\footnotesize
\setlength{\tabcolsep}{5pt}
\caption{Cross entropy comparison across conditioning modalities and test sets (median BCE). Lower is better. Best results per column are \textbf{bolded}.}
\label{tab:combined_conditioning}
\begin{tabular}{l ccc c ccc c ccc}
\toprule
 & \multicolumn{3}{c}{\textbf{Sensitivity}} & & \multicolumn{3}{c}{\textbf{Pseudo-Sen. (Velocity)}} & & \multicolumn{3}{c}{\textbf{Pressure}} \\
\cmidrule(lr){2-4} \cmidrule(lr){6-8} \cmidrule(lr){10-12}
\textbf{Model} & \textbf{ID} & \textbf{OOD-M} & \textbf{OOD-H} & & \textbf{ID} & \textbf{OOD-M} & \textbf{OOD-H} & & \textbf{ID} & \textbf{OOD-M} & \textbf{OOD-H} \\
\midrule
DiT   & $0.17$ & $0.36$ & $0.64$ & & $0.19$ & $0.49$ & $0.91$ & & $5.21$ & $5.25$ & $5.24$ \\
UDiT  & $0.08$ & $0.15$ & $0.23$ & & $0.08$ & $0.20$ & $\mathbf{0.36}$ & & $0.24$ & $1.01$ & $1.67$ \\
PDE-T & $0.08$ & $0.22$ & $0.41$ & & $0.07$ & $0.29$ & $0.58$ & & $0.36$ & $1.91$ & $3.03$ \\
\textbf{Ours} & $\mathbf{0.04}$ & $\mathbf{0.12}$ & $\mathbf{0.22}$ & & $\mathbf{0.04}$ & $\mathbf{0.15}$ & $0.38$ & & $\mathbf{0.07}$ & $\mathbf{0.23}$ & $\mathbf{0.46}$ \\
\bottomrule
\end{tabular}
\end{table*}
\subsection{Conditioning Signal Ablation}
\label{sec:conditioning_ablation}

Table~\ref{tab:combined_conditioning} validates our theoretical predictions by comparing cross entropy across three conditioning modalities: sensitivity, velocity (pseudo-sensitivity), and pressure. The DPI-derived information ordering is evident consistently across all architectures; Every model achieves lowest CE under sensitivity conditioning and highest under pressure, with velocity intermediate. This confirms that the conditioning hierarchy is fundamental to the problem structure, not an artifact of any particular architecture.

\subsection{Topology Control via Sensitivity Manipulation}
\label{sec:sensitivity_manipulation}

A practical advantage of sensitivity-based conditioning is the ability to manipulate the input field to enforce geometric constraints on the output topology. By masking regions of the sensitivity field (setting them to values that discourage material placement), we can block specific areas from the generated design. Figure~\ref{fig:manip} illustrates this capability across four test cases: a circular exclusion zone is applied to the sensitivity field, and the model generates topologies that route flow around the blocked region while maintaining connectivity. This is particularly useful when other components occupy portions of the design domain, or when manufacturing constraints require keep-out regions.

\begin{figure}[tb]
    \centering
    \includegraphics[width=0.85\linewidth]{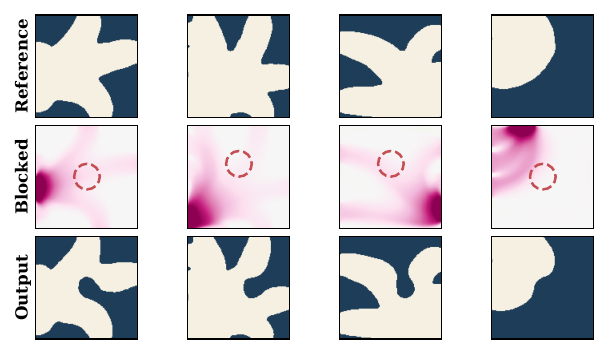}
    \caption{Topology control via sensitivity masking. Top row: reference topologies. Middle row: sensitivity fields with circular exclusion zones (dashed red). Bottom row: generated topologies that successfully avoid the blocked region while preserving flow connectivity.}
    \label{fig:manip}
\end{figure}

\subsection{Volume Fraction Constraints}
\label{sec:volume_constraint}

The Bernoulli flow formulation naturally supports volume fraction constraints. We propose a confidence-based progressive pruning strategy that enforces the constraint during sampling rather than as a post-processing step. At each diffusion timestep, we remove low-confidence material before the next iteration, allowing the model to adapt its predictions to the pruned state. Algorithm~\ref{alg:progressive_vf} details this procedure.

\begin{algorithm}[tb]
\caption{Confidence-based progressive volume constraint.}
\label{alg:progressive_vf}
\begin{algorithmic}[1]
\REQUIRE Model $f_\theta$; volume limit $V_{\max}$; steps $T$
\STATE Sample $\mathbf{x}_{t_T} \sim \text{Bernoulli}(0.5)$
\FOR{$\tau = T, T{-}1, \ldots, 1$}
    \STATE $\ell_\tau = f_\theta(\mathbf{x}_{t_\tau}, t_\tau)$ \hfill $\triangleright$ Raw logits
    \STATE $p_\tau = \sigma(\ell_\tau / \kappa)$ \hfill $\triangleright$ Probabilities
    \STATE $\mathbf{x}_{t_{\tau-1}} \sim \text{Bernoulli}(p_\tau)$
    \IF{$\text{mean}(\mathbf{x}_{t_{\tau-1}}) > V_{\max}$}
        \STATE $\theta_\tau = Q_{1-V_{\max}}(\ell_\tau \odot \mathbf{x}_{t_{\tau-1}})$ \hfill $\triangleright$ Confidence threshold
        \STATE $\mathbf{x}_{t_{\tau-1}} = \mathbf{x}_{t_{\tau-1}} \odot \mathbf{1}[\ell_\tau > \theta_\tau]$ \hfill $\triangleright$ Prune low-confidence
    \ENDIF
\ENDFOR
\STATE \textbf{return} $\mathbf{x}_{t_0}$
\end{algorithmic}
\end{algorithm}

The key insight is that pruning occurs before the next forward pass, so the model observes the constrained state and can redistribute material to maintain connectivity. Figure~\ref{fig:vf} demonstrates this mechanism across four test cases: starting from unconstrained topologies, we progressively tighten the volume budget to $90\%$, $80\%$, and $70\%$ of the original. The model preserves primary flow channels while trimming secondary branches, achieving the target volume fraction without disconnecting the flow path.

\begin{figure}[tb]
    \centering
    \includegraphics[width=0.85\linewidth]{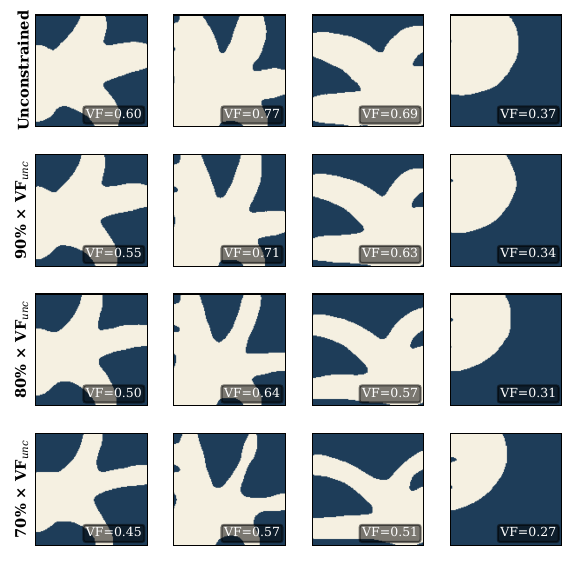}
    \caption{Confidence-based progressive volume constraint. Each column shows a different test case; rows correspond to decreasing volume budgets (unconstrained, $90\%$, $80\%$, $70\%$ of $\mathrm{VF}_{\mathrm{unc}}$). The resulting volume fractions are annotated on each panel. The model adapts to maintain flow connectivity under tighter constraints.}
    \label{fig:vf}
\end{figure}

\subsection{Greedy Terminal Step for Clean Topologies}
\label{sec:greedy_sampling}

The standard BFM sampling procedure~\cite{gat2024bernoulli} draws independent Bernoulli samples at each pixel in the final step, which can introduce salt-and-pepper noise in the generated topology (Figure~\ref{fig:sample}, top row). While acceptable for image generation, such artifacts are problematic for topology optimization: isolated pixels cause mesh generation failures and numerical instabilities in downstream simulations.

We address this by replacing the stochastic final step with a deterministic greedy projection. As shown in Algorithm~\ref{alg:greedy_sampling}, after $T-1$ stochastic transitions, the final state is obtained by thresholding the predicted probability field at $0.5$ rather than sampling. This produces binary topologies free of isolated pixels while preserving the probabilistic exploration during earlier diffusion steps.

\begin{algorithm}[tb]
\caption{Greedy terminal step sampling.}
\label{alg:greedy_sampling}
\begin{algorithmic}[1]
\REQUIRE Trained model $f_\theta$; diffusion steps $T$; temperature $\kappa$
\STATE Sample $\mathbf{x}_{t_T} \sim \text{Bernoulli}(0.5)$
\FOR{$\tau = T, T{-}1, \ldots, 2$}
    \STATE Predict $p_\theta(\mathbf{x}_{t_{\tau-1}}) = \sigma(f_\theta(\mathbf{x}_{t_\tau}, t_\tau) / \kappa)$
    \STATE Sample $\mathbf{x}_{t_{\tau-1}} \sim \text{Bernoulli}(p_\theta(\mathbf{x}_{t_{\tau-1}}))$
\ENDFOR
\STATE Predict $p_\theta(\mathbf{x}_{t_0}) = \sigma(f_\theta(\mathbf{x}_{t_1}, t_1) / \kappa)$
\STATE $\hat{\mathbf{x}}_{t_0} = \mathbf{1}[p_\theta(\mathbf{x}_{t_0}) > 0.5]$ \hfill $\triangleright$ Greedy thresholding
\STATE \textbf{return} $\hat{\mathbf{x}}_{t_0}$
\end{algorithmic}
\end{algorithm}

\begin{figure}[tb]
    \centering
    \includegraphics[width=0.85\linewidth]{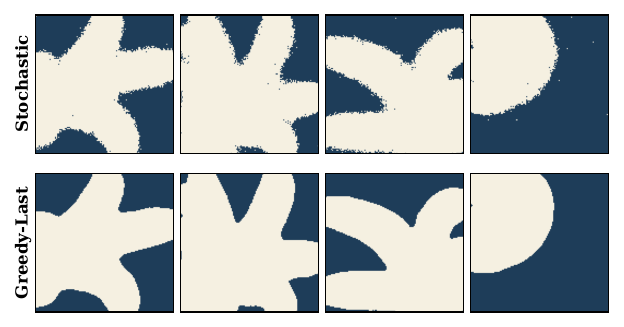}
    \caption{Comparison of sampling strategies across four test cases. Top row: fully stochastic BFM sampling produces salt-and-pepper artifacts along channel boundaries. Bottom row: greedy terminal step yields clean, simulation-ready topologies with identical macro-structure.}
    \label{fig:sample}
\end{figure}

\subsection{Sufficiency of Initial Sensitivity $S_0$}
\label{app:auc_sufficiency}

A key modeling assumption is that the initial sensitivity $S_0$ captures sufficient information for predicting the final topology. We validate this using the area under the ROC curve (AUC): the probability that a solid element in $\rho^*$ ranks above a void element under a given sensitivity field. Table~\ref{tab:auc} reports AUC for initial ($S_0$) and converged ($S^*$) sensitivities across four PDE families, including two newly generated (heat conduction and eigenfrequency).

\begin{table}[tb]
\centering
\small
\caption{AUC of initial vs.\ converged sensitivities across PDE families. $S_0$ achieves AUC comparable to $S^*$ in all cases, confirming practical sufficiency. AUC(vf) shows that volume-fraction thresholding alone is far weaker.}
\label{tab:auc}
\begin{tabular}{lcccc}
\toprule
\textbf{Problem} & \textbf{AUC($S_0$)} & \textbf{AUC($S^*$)} & $\boldsymbol{\Delta}$ & \textbf{AUC(vf)} \\
\midrule
Heat conduction & $0.972$ & $0.986$ & $+0.014$ & $0.814$ \\
NS 3-outlet & $0.993$ & $1.000$ & $+0.007$ & $0.799$ \\
Eigenfrequency & $0.957$ & $0.868$ & $-0.089$ & $0.730$ \\
Compliance & $0.915$ & $0.918$ & $+0.003$ & $0.731$ \\
\bottomrule
\end{tabular}
\end{table}

AUC($S_0$) $> 0.91$ in all cases. If $S_0$ were substantially insufficient, AUC should consistently increase from initialization to convergence; this is not observed. Notably, for eigenfrequency, AUC \emph{decreases} from $S_0$ to $S^*$, meaning the initial sensitivity is more discriminative of $\rho^*$ than the converged field. The AUC(vf) column confirms that rank sufficiency does not imply trivial extractability: volume-fraction thresholding alone yields substantially lower AUC, and a learned surrogate is needed to resolve spatial topology.

\subsection{Quantitative Validation of Pseudo-Sensitivity}
\label{app:pseudo_validation}

We validate the pseudo-sensitivity approximation $S \propto -\|\vv\|^2$ quantitatively across the full CFD dataset. For all 1-, 2-, and 3-outlet configurations, the Pearson correlation between $\|\vv\|^2$ and the true sensitivity yields a mean $r > 0.95$, with fewer than 2.2\% of samples falling below $r = 0.80$.

Figure~\ref{fig:re_correlation} shows the Pearson correlation binned by Reynolds number. At low $\mathrm{Re}$, correlation is highest ($r > 0.98$), consistent with the Stokes symmetry argument. A mild degradation occurs in the transitional $\mathrm{Re}$ regime, as predicted by the theory. At high $\mathrm{Re}$, correlation recovers and variance collapses as the adjoint scaling $\ww \to -\vv$ takes hold. Increasing inlet count further improves correlation even at moderate $\mathrm{Re}$, as distributed inflow reduces recirculation and adjoint transport effects.

\begin{figure}[tb]
    \centering
    \includegraphics[width=0.85\linewidth]{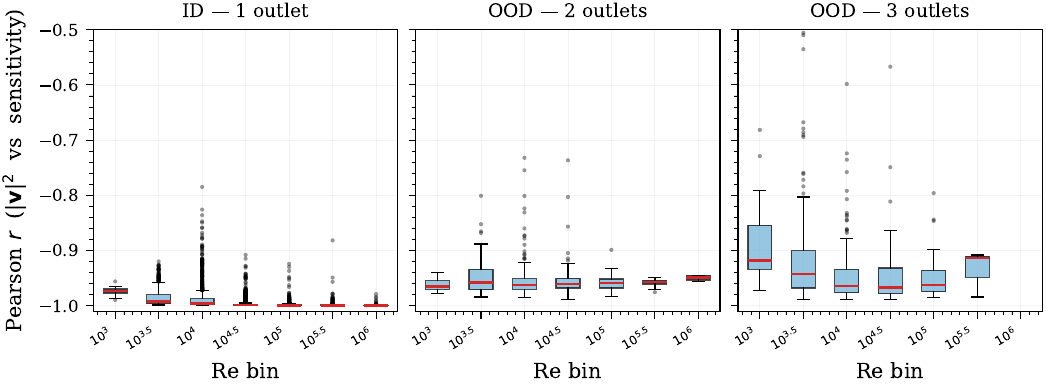}
    \caption{Pearson correlation between pseudo-sensitivity ($\|\vv\|^2$) and true sensitivity, binned by Reynolds number. The approximation holds strongly at both low and high $\mathrm{Re}$, with mild degradation confined to the transitional regime.}
    \label{fig:re_correlation}
\end{figure}

\section{Datasets Details}
\label{app_datasets}

\subsection{CFD Topology Optimization Dataset}
\label{app:cfd_dataset}

\subsubsection{Domain and Boundary Conditions}
We consider a 2D square domain $\Omega = [0, L]^2$ with side length $L = 0.1\,\mathrm{m}$.
Inlets and outlets are placed along the domain boundary with fixed aperture size $w = 0.02\,\mathrm{m}$.
Inlets are assigned a uniform velocity profile with magnitude $u_{\mathrm{in}} = 1\,\mathrm{m/s}$ normal to the boundary.
Outlets are modeled as pressure outlets with zero gauge pressure.
All remaining boundary segments are treated as no-slip walls.

\subsubsection{Governing Equations and Solver}
All simulations use steady-state Reynolds-Averaged Navier--Stokes (RANS) equations with the realizable $k$--$\varepsilon$ turbulence model, solved in STAR-CCM+ \cite{starccm}.
The flow is treated as incompressible with constant density $\rho$ and dynamic viscosity $\mu$ per instance.
Convergence is declared when residuals for continuity, momentum, $k$, and $\varepsilon$ drop below $10^{-4}$.

\subsubsection{Topology Optimization Formulation}
The optimization objective minimizes pressure drop plus a volume penalty:
\begin{equation}
\mathcal{J} = \Delta p + \lambda \int_\Omega \chi \, dV,
\label{eq:cfd_objective}
\end{equation}
where $\Delta p$ is the mass-flow-averaged total pressure difference between inlet and outlet boundaries, $\chi \in \{0,1\}$ is the material indicator (fluid vs.\ solid), and $\lambda$ is a scalar weight controlling the volume fraction.

The pressure drop is computed as
\begin{equation}
\Delta p = \left[ \frac{\sum_f |\dot{m}_f| \, p_{t,f}}{\sum_f |\dot{m}_f|} \right]_{\mathrm{inlet}} - \left[ \frac{\sum_f |\dot{m}_f| \, p_{t,f}}{\sum_f |\dot{m}_f|} \right]_{\mathrm{outlet}},
\end{equation}
where $\dot{m}_f$ is the mass flow rate through face $f$ and $p_{t,f}$ is the absolute total pressure at that face.

Solid regions are enforced via Brinkman penalization, which adds a drag term to the momentum equations:
\begin{equation}
\frac{\partial (\rho \vv)}{\partial t} + \nabla \cdot (\rho \vv \otimes \vv) = -\nabla \cdot \bm{\sigma} - \alpha (1 - \chi) \vv,
\end{equation}
where $\alpha$ is the penalization magnitude (set to $5\times10^4$ in our simulations) and $\chi = 1$ in fluid regions, $\chi = 0$ in solid regions.
This formulation models solid material as a porous medium with vanishing permeability.

Optimization is performed using Adam with learning rate $10$, $\beta_1 = 0.9$, $\beta_2 = 0.999$, for a maximum of 30 iterations per instance or reaching $0.01\%$ material change.

\subsubsection{Dataset Splits and Sample Counts}
We construct distribution shifts by varying the number of outlets:
\begin{center}
\begin{tabular}{lcccc}
\toprule
\textbf{Split} & $N_{\mathrm{out}}$ & $N_{\mathrm{in}}$ range & \textbf{Training} & \textbf{Test} \\
\midrule
ID             & 1 & 1--11 & 10,000 & 1,000 \\
OOD-medium     & 2 & 1--10 & ---    & 500   \\
OOD-hard       & 3 & 1--9  & ---    & 500   \\
\bottomrule
\end{tabular}
\end{center}
Inlet and outlet positions are sampled uniformly along the boundary edges, subject to a minimum spacing constraint of $0.015\,\mathrm{m}$ to avoid overlapping apertures.

\subsubsection{Fluid Property Sampling}
For each instance, we independently sample dynamic viscosity and density:
\begin{align}
\mu &\sim \mathrm{Unif}(10^{-5}, 10^{-3})\ \mathrm{Pa \cdot s}, \\
\rho &\sim \mathrm{Unif}(0.5, 10.0)\ \mathrm{kg/m^3}.
\end{align}
This range spans from low-viscosity gases to moderately viscous liquids, ensuring diversity in flow regimes.

\subsubsection{Reynolds Number Ranges}
\label{app:reynolds}
We characterize flow regimes via the Reynolds number
\begin{equation}
\mathrm{Re} = \frac{\rho \, U \, L}{\mu}.
\end{equation}
To define a consistent velocity scale across configurations with varying inlet/outlet counts, we use mass conservation.
Assuming incompressible flow and equal aperture sizes, the mean outlet velocity satisfies
\begin{equation}
U_{\mathrm{out}} \approx \frac{N_{\mathrm{in}}}{N_{\mathrm{out}}} \, u_{\mathrm{in}}.
\end{equation}
The resulting per-outlet Reynolds number ranges are:
\begin{center}
\begin{tabular}{lccc}
\toprule
\textbf{Split} & $N_{\mathrm{out}}$ & $N_{\mathrm{in}}$ range & $\mathrm{Re}_{\mathrm{out}}$ range \\
\midrule
ID         & 1 & 1--11 & $[50,\ 1.1 \times 10^{6}]$ \\
OOD-medium & 2 & 1--10 & $[25,\ 5 \times 10^{5}]$   \\
OOD-hard   & 3 & 1--9  & $[16.7,\ 3 \times 10^{5}]$ \\
\bottomrule
\end{tabular}
\end{center}
These ranges span laminar to fully turbulent regimes, with the turbulence model active in all simulations.

\subsubsection{Sensitivity Computation}
Adjoint sensitivities are computed using the discrete adjoint solver in STAR-CCM+.
We employ the frozen turbulence assumption, treating eddy viscosity as constant with respect to design perturbations.
This approximation is standard for RANS-based topology optimization and yields stable gradients across the dataset.
The initial sensitivity field $S_0$ is evaluated at a uniform density $\rho_{\mathrm{init}} = 1.0$ before optimization begins.

\subsubsection{Data Fields Stored}
For each sample, we store:
\begin{itemize}
    \item Binary optimal topology $\rho^\star \in \{0,1\}^{128 \times 128}$
    \item Initial sensitivity field $S_0 \in \mathbb{R}^{128 \times 128}$
    \item Velocity magnitude field $\|\vv\| \in \mathbb{R}^{128 \times 128}$
    \item Vector velocity field $v \in \mathbb{R}^{2\times 128 \times 128}$
    \item Pressure field $p \in \mathbb{R}^{128 \times 128}$
    \item Boundary condition mask indicating inlet/outlet/wall locations
    \item Scalar metadata: $\mu$, $\rho$, $N_{\mathrm{in}}$, $N_{\mathrm{out}}$, final objective value
\end{itemize}
All fields are interpolated to a uniform $128 \times 128$ Cartesian grid.

\subsection{Visualized Samples}
\label{app:dataset_visualization}

Figures~\ref{fig:app_singleoutlet}--\ref{fig:app_threeoutlet} show representative samples from each dataset split, illustrating the diversity of boundary configurations and corresponding optimal topologies. Each visualization displays the velocity magnitude field overlaid with the optimized material distribution, where white regions indicate solid material and colored regions show fluid flow. The sensitivity field $S_0$ used for conditioning is computed at uniform initial density before optimization.

\begin{figure}[tb]
    \centering
    \includegraphics[width=0.75\linewidth]{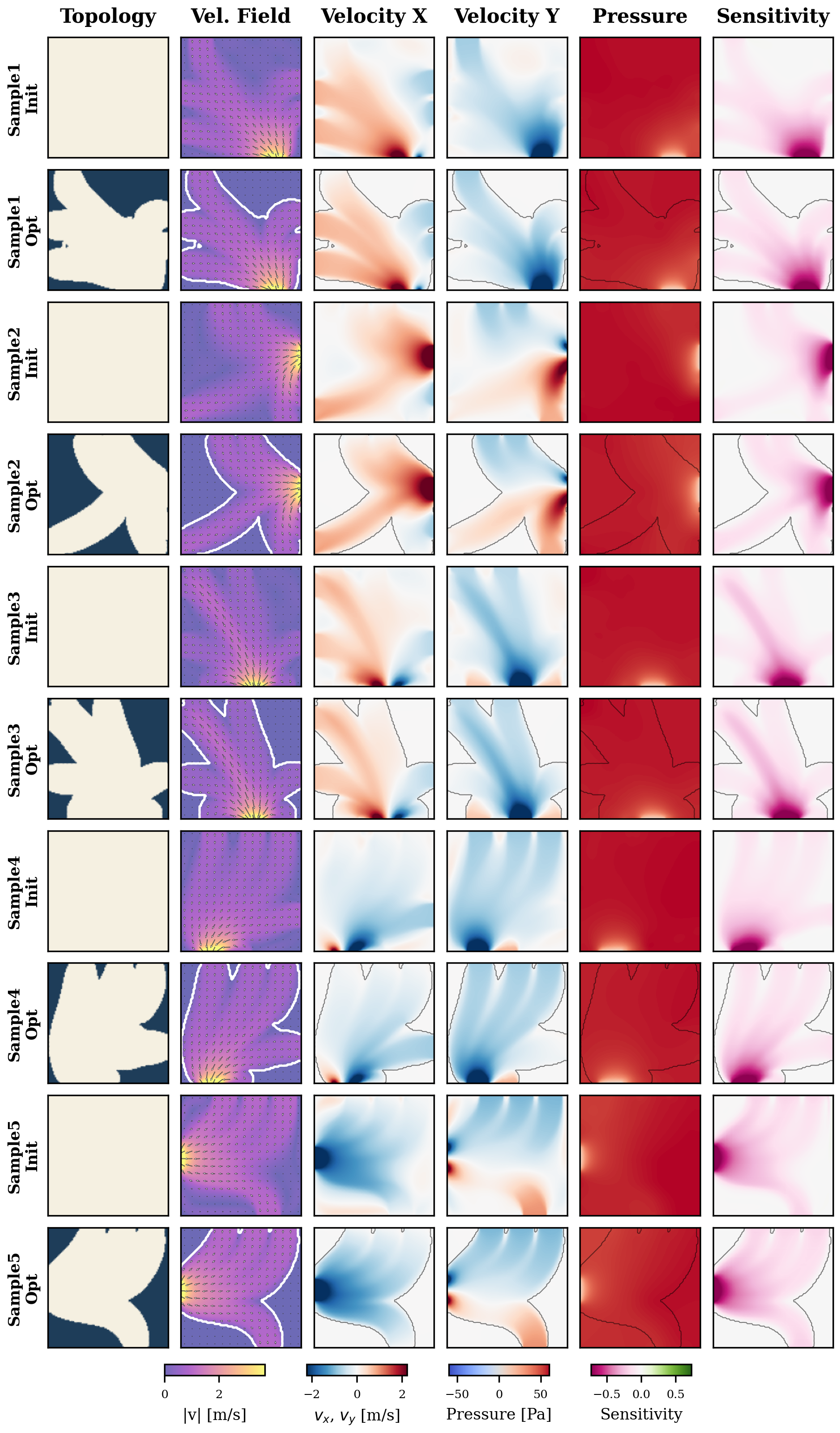}
    \caption{Representative samples from the in-distribution (ID) training set with a single outlet. Inlet positions vary along the domain boundary, producing diverse flow patterns and optimal topologies that route flow from inlet(s) to the single outlet while minimizing pressure drop.}
    \label{fig:app_singleoutlet}
\end{figure}

\begin{figure}[tb]
    \centering
    \includegraphics[width=0.75\linewidth]{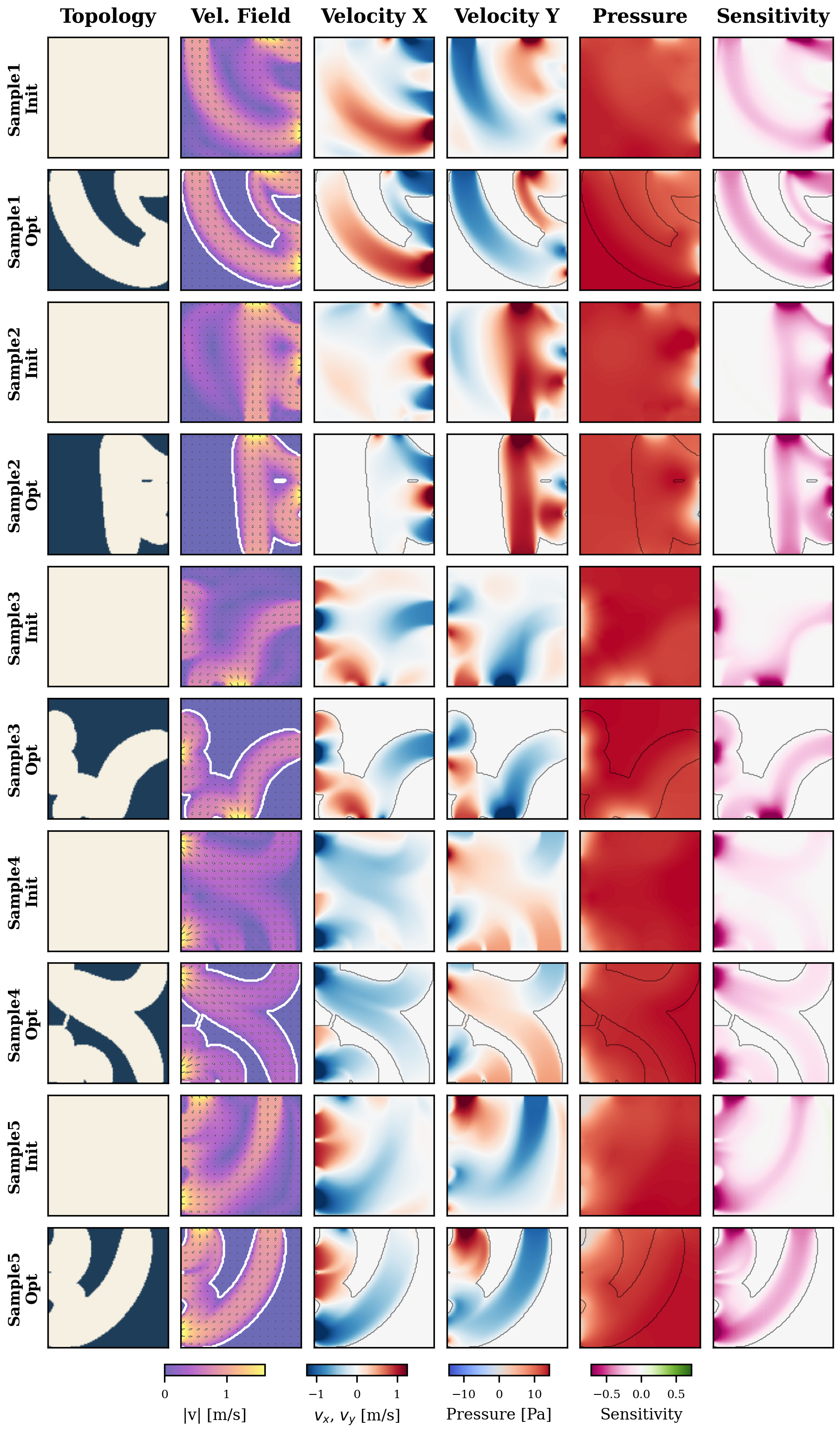}
    \caption{Representative samples from the OOD-medium test set with two outlets. The model must generalize to splitting flow between multiple outlets---a configuration unseen during training. Note the branching channel structures that emerge to efficiently distribute flow.}
    \label{fig:app_twooutlet}
\end{figure}

\begin{figure}[tb]
    \centering
    \includegraphics[width=0.75\linewidth]{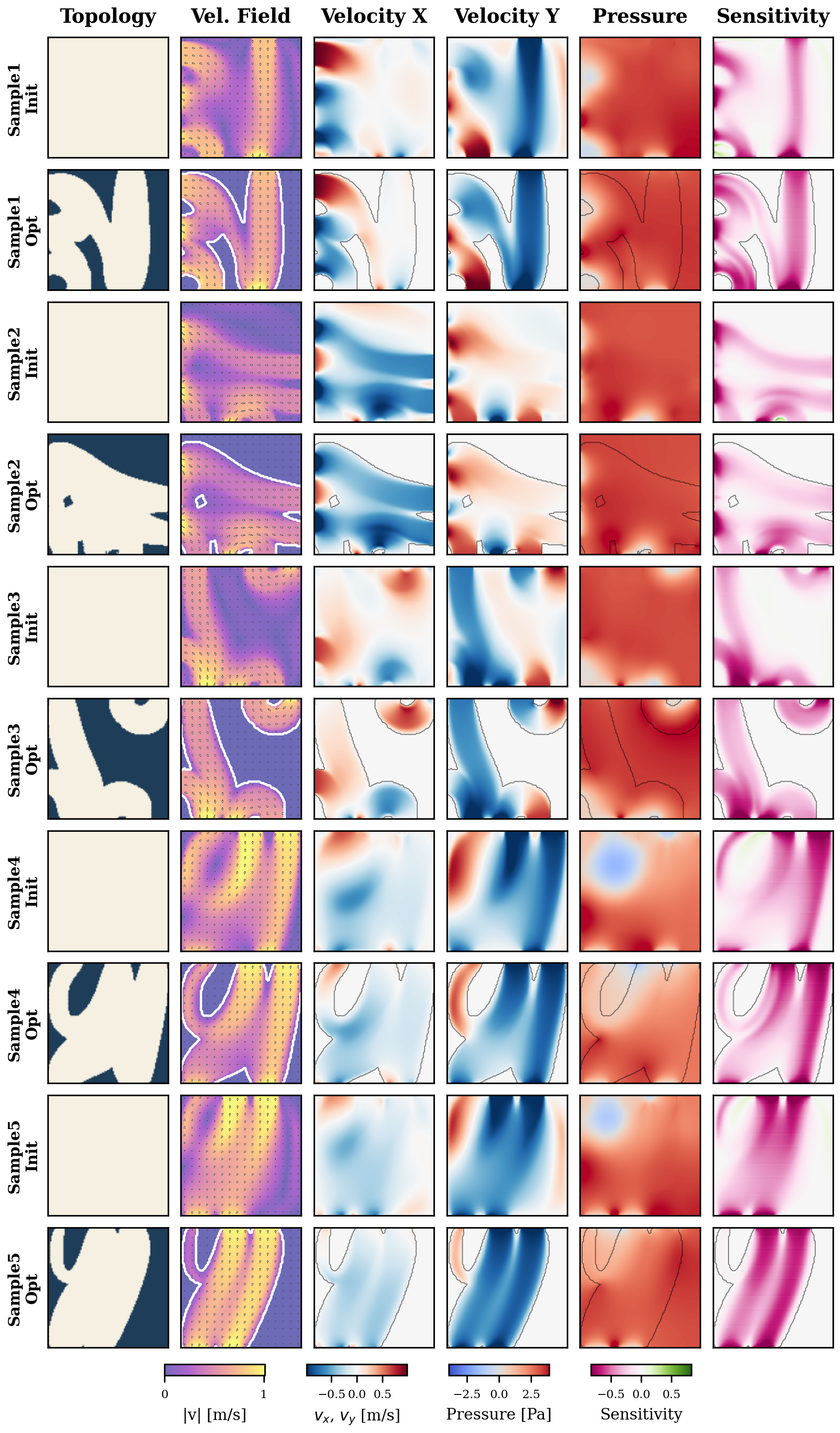}
    \caption{Representative samples from the OOD-hard test set with three outlets. This configuration represents the most challenging distribution shift, requiring complex multi-branch topologies to route flow from varying inlet configurations to three distinct outlet locations.}
    \label{fig:app_threeoutlet}
\end{figure}

\subsection{Dataset Visualization}
\label{appendix:data_visualization}

Figures~\ref{fig:structural_test} and~\ref{fig:structural_ood} present representative samples from the in-distribution and out-of-distribution (OOD) test sets, respectively. Each row shows:

\begin{itemize}
    \item \textbf{BC:} Boundary conditions---red (pinned), green (roller-$y$), blue (roller-$x$).
    \item \textbf{Applied Load:} Force vectors shown as red arrows.
    \item \textbf{Optimal Topology:} Converged material distribution (light = void, dark = solid).
    \item \textbf{Sensitivity:} First-iteration sensitivity field ($\log_{10}$ scale).
    \item \textbf{Strain Energy:} Strain energy density ($\log_{10}$ scale).
    \item \textbf{Displacement:} Normalized displacement magnitude.
\end{itemize}

\begin{figure}[t]
    \centering
    \includegraphics[width=0.75\textwidth]{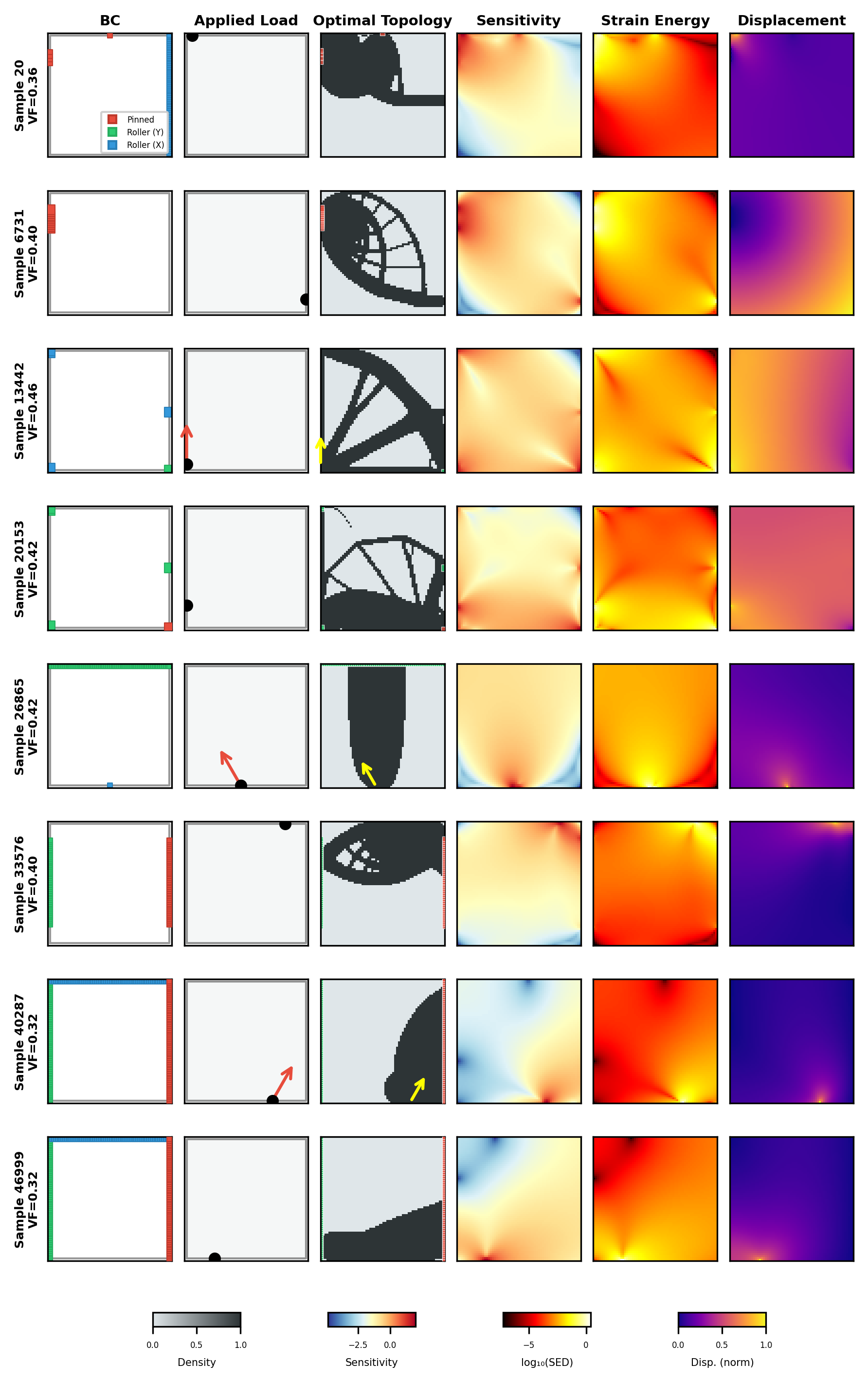}
    \caption{\textbf{In-Distribution Set.} Representative samples showing BC, loads, optimal topologies, sensitivity, strain energy, and displacement magnitude.}
    \label{fig:structural_test}
\end{figure}

\begin{figure}[t]
    \centering
    \includegraphics[width=0.75\textwidth]{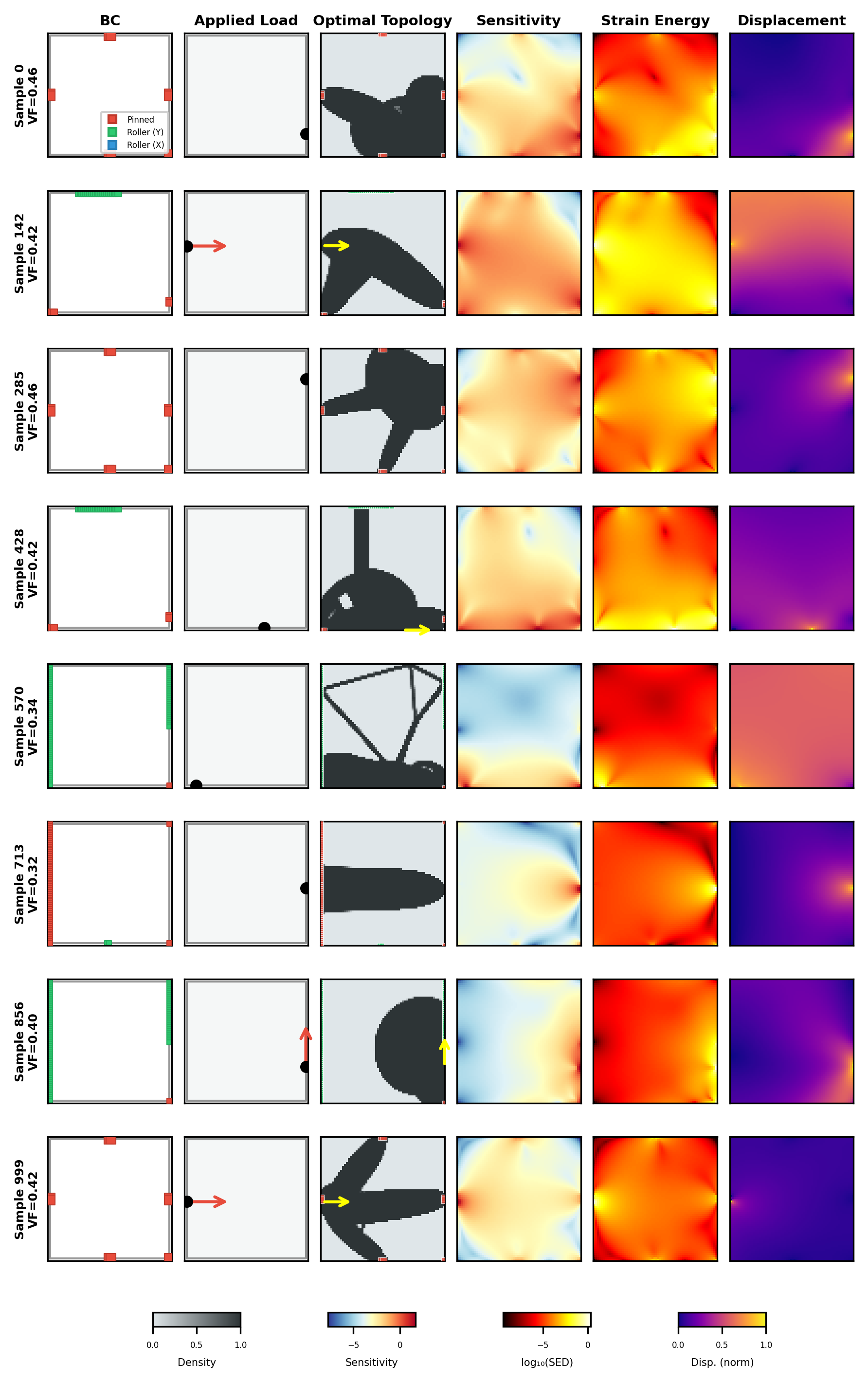}
    \caption{\textbf{Out-of-Distribution Test Set.} Samples with novel boundary condition configurations for evaluating generalization.}
    \label{fig:structural_ood}
\end{figure}
\end{mohammad}
\end{document}